\documentclass[11pt]{article}
\usepackage[top=1in, bottom=1in, left=1in, right=1in]{geometry}
\usepackage{amssymb,mathrsfs,amsmath,amsthm}
\usepackage{wrapfig}
\usepackage{hyperref}

\usepackage{amsmath,amssymb,amsbsy,amsfonts,amscd,bm,url,color,latexsym}
\usepackage{paralist}
\usepackage[dvipsnames]{xcolor}
\usepackage{color}
\usepackage{graphicx}
\graphicspath{{./figs/}}
\usepackage{algorithm}
\usepackage{algorithmic}
\usepackage{comment}
\usepackage{multirow}
\usepackage{enumitem}
\usepackage{fancyhdr}
\usepackage{cite}
\usepackage{cleveref}
\usepackage{subcaption}
\usepackage{authblk}
\newtheorem{prop}{Proposition}

\newcommand{\R}{\mathbb{R}}

\newcommand{\e}{\begin{equation}}
\newcommand{\ee}{\end{equation}}
\newcommand{\en}{\begin{equation*}}
\newcommand{\een}{\end{equation*}}
\newcommand{\eqn}{\begin{eqnarray}}
\newcommand{\eeqn}{\end{eqnarray}}
\newcommand{\bmat}{\begin{bmatrix}}
\newcommand{\emat}{\end{bmatrix}}

\DeclareMathAlphabet\mathbfcal{OMS}{cmsy}{b}{n}

\newcommand{\mb}{\mathbf}
\newcommand{\mc}{\mathcal}

\newcommand{\bb}{\mathbb}

\newcommand{\vct}[1]{\boldsymbol{#1}}
\newcommand{\mtx}[1]{\boldsymbol{#1}}

\newcommand{\trace}{\operatorname{trace}}

\newcommand{\diag}{\operatorname{diag}}

\def \st {\operatorname*{s.t.\ }}

\newcommand{\wt}{\widetilde}
\newcommand{\ol}{\overline}
\newcommand{\NC}{$\mc {NC}$}

\newcommand{\norm}[2]{\left\| #1 \right\|_{#2}}

\newcommand{\innerprod}[2]{\left\langle #1,  #2 \right\rangle}

\newcommand{\vy}{\vct{y}}

\newcommand{\mH}{\mtx{H}}

\newcommand{\mW}{\mtx{W}}

\newcommand{\mDelta}{\mtx{\Delta}}

\graphicspath{{./figs/}}

\newlength{\imgwidth}
\newboolean{twoColVersion}
\setboolean{twoColVersion}{false}
\newcommand{\twoCol}[2]{\ifthenelse{\boolean{twoColVersion}} {#1} {#2} }

\usepackage{tikz}
\usepackage{tikz-3dplot}
\usepackage{pgfplots}
\usepgfplotslibrary{patchplots}
\usetikzlibrary{patterns, positioning, arrows}
\pgfplotsset{compat=1.15}

\usepackage{tgpagella}
\usepackage[utf8]{inputenc} 
\usepackage[T1]{fontenc}    
\usepackage{url}

\hypersetup{
    colorlinks=true,%
    citecolor=blue,%
    filecolor=blue,%
    linkcolor=blue,%
    urlcolor=blue
}
\usepackage[toc, page]{appendix}

\newtheorem{theorem}{Theorem}[section]
\newtheorem{lemma}[theorem]{Lemma}

\renewcommand{\mathbf}{\boldsymbol}

\newcommand{ \Brac }[1]{\left\lbrace #1 \right\rbrace}
\newcommand{ \brac }[1]{\left[ #1 \right]}
\newcommand{ \paren }[1]{ \left( #1 \right) }

\DeclareMathOperator{\ddiag}{ddiag}
\DeclareMathOperator{\grad}{grad}
\DeclareMathOperator{\Hess}{Hess}

\newcommand{\mr}{\mathrm}

\newcommand{\Kmod}{K\mbox{\textup{ mod }} 2}

\title{Neural Collapse with Normalized Features:\\ A Geometric Analysis over the Riemannian Manifold}

\author[1]{Can Yaras\footnote{The first two authors contributed to this work equally.}}

\newcommand\CoAuthorMark{\footnotemark[\arabic{footnote}]}
\author[1]{Peng Wang\protect\CoAuthorMark}
\author[2]{Zhihui Zhu}
\author[1]{Laura Balzano}
\author[1]{Qing Qu}
\affil[1]{Department of Electrical Engineering \& Computer Science, University of Michigan} 
\affil[2]{Department of Computer Science \& Engineering, Ohio State University}

\date{}
\begin{document}

\maketitle

\begin{abstract}
When training overparameterized deep networks for classification tasks, it has been widely observed that the learned features exhibit a so-called ``neural collapse'' phenomenon. More specifically, for the output features of the penultimate layer, for each class the within-class features converge to their means, and the means of different classes exhibit a certain tight frame structure, which is also aligned with the last layer's classifier. As feature normalization in the last layer becomes a common practice in modern representation learning, in this work we theoretically justify the neural collapse phenomenon for normalized features. Based on an unconstrained feature model, we simplify the empirical loss function in a multi-class classification task into a nonconvex optimization problem over the Riemannian manifold by constraining all features and classifiers over the sphere. In this context, we analyze the nonconvex landscape of the Riemannian optimization problem over the product of spheres, showing a benign global landscape in the sense that the only global minimizers are the neural collapse solutions while all other critical points are strict saddles with negative curvature. Experimental results on practical deep networks corroborate our theory and demonstrate that better representations can be learned faster via feature normalization. The code for our experiments can be found at \url{https://github.com/cjyaras/normalized-neural-collapse}.

\end{abstract}

\section{Introduction}\label{sec:intro}
Despite the tremendous success of deep learning in engineering and scientific applications over the past decades, the underlying mechanism of deep neural networks (DNNs) still largely remains mysterious. Towards the goal of understanding the learned deep representations, a recent line of seminal works \cite{papyan2020prevalence,han2021neural,fang2021exploring,zhu2021geometric,zhou2022optimization} presents an intriguing phenomenon that persists across a range of canonical classification problems during the terminal phase of training. Specifically, it has been widely observed that last-layer features (i.e., the output of the penultimate layer) and last-layer linear classifiers of a trained DNN exhibit 
simple but elegant mathematical structures, in the sense that
\begin{itemize}[leftmargin=*]
     \item \textbf{(NC1) Variability Collapse:} the individual features of each class concentrate to their class-means.
    \item \textbf{(NC2) Convergence to Simplex ETF:} the class-means have the same length and are maximally distant. In other words, they form a Simplex Equiangular Tight Frame (ETF).
    \item \textbf{(NC3) Convergence to Self-Duality:} the last-layer linear classifiers perfectly match their class-means.
\end{itemize}
Such a phenomenon is referred to as \emph{Neural Collapse} (\NC) \cite{papyan2020prevalence}, which has been shown empirically to persist across a broad range of canonical classification problems, on different loss functions (e.g., cross-entropy (CE)~\cite{papyan2020prevalence,zhu2021geometric,fang2021exploring}, mean-squared error (MSE)~\cite{tirer2022extended,zhou2022optimization}, and supervised contrasive (SC) losses \cite{graf2021dissecting}), on different neural network architectures (e.g., VGG~\cite{simonyan2014very}, ResNet \cite{he2016deep}, and DenseNet \cite{huang2017densely}), and on a variety of standard datasets (such as MNIST \cite{lecun2010mnist}, CIFAR \cite{krizhevsky2009learning}, and ImageNet \cite{deng2009imagenet}, etc). Recently, in independent lines of research, many works are devoted to learning  maximally compact and separated features; see, e.g., \cite{wen2016discriminative,liu2016large,ranjan2017l2,wang2017normface,liu2017sphereface,wang2018cosface,deng2019arcface,pernici2019maximally,pernici2021regular}.
This has also been widely demonstrated in a number of recent works \cite{cohen2020separability,mamou2020emergence,doimo2020hierarchical,naitzat2020topology,recanatesi2019dimensionality,frosst2019analyzing,hofer2020topologically}, including state-of-the-art natural language models (such as BERT, RoBERTa, and GPT) \cite{mamou2020emergence}. 

\begin{figure}[t]
            \centering
            \includegraphics[width=0.8\textwidth]{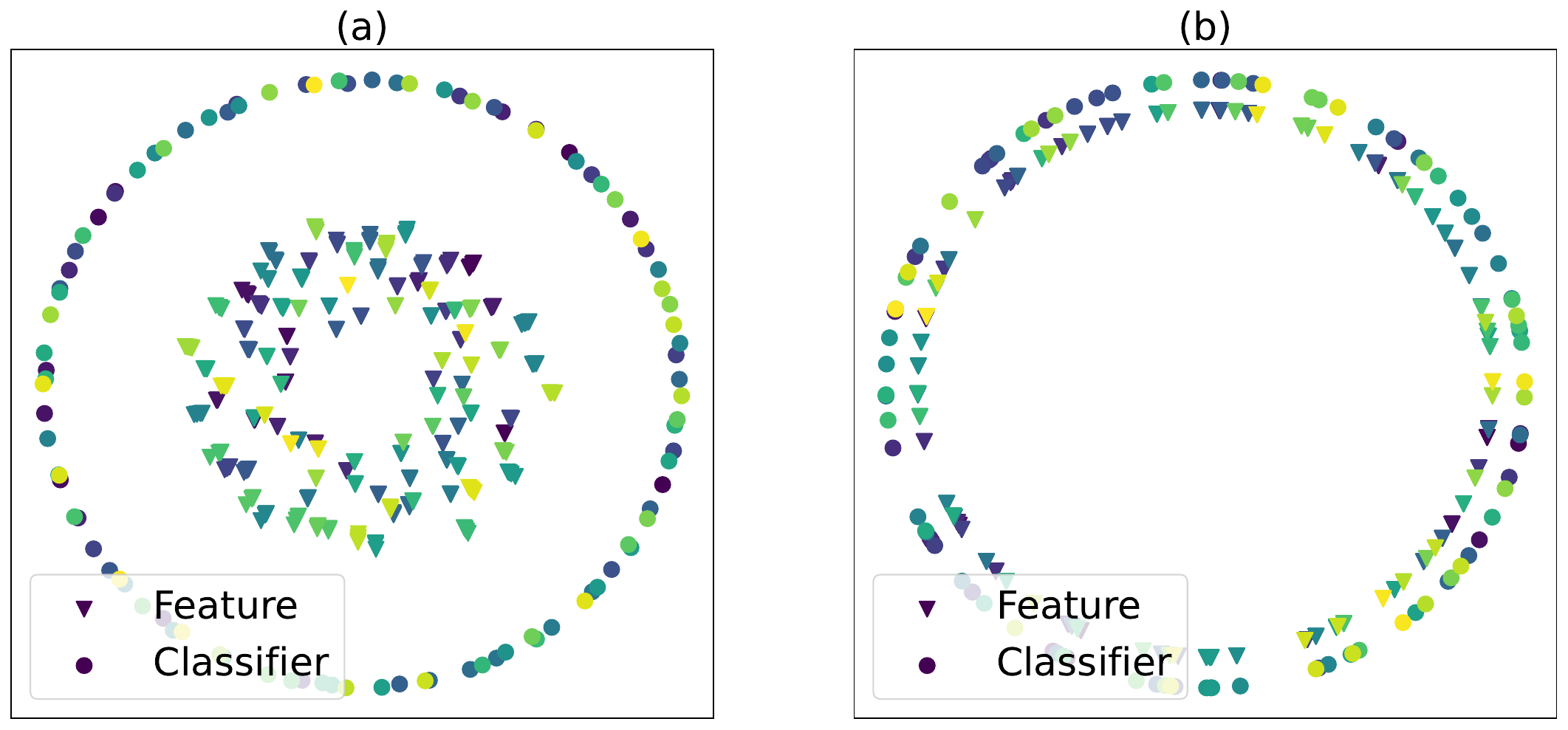}
            \caption{\textbf{Comparison of features found with and without normalization.} $K=100$ classes with $n=5$ samples per class. Features and classifiers are found through optimizing the cross-entropy loss with a UFM, where features are embedded in 2-dimensional space, i.e., $d=2$. (a) No normalization of features or classifiers. (b) Features and classifiers are constrained to the unit sphere (features are scaled down for visualization purposes). }
            \label{fig:low_dim_features}
\end{figure}
\begin{table}[t]
    \centering
    \renewcommand{\arraystretch}{1.1}
    \scalebox{1.1}{
    \begin{tabular}{|l|c|c|} \hline
         & \textbf{Average CE Loss} & \textbf{Average Accuracy} \\ \hline
        No Normalization & $1.63 \pm 0.03$ & $49.9\% \pm 2.39\%$ \\ \hline
        Normalization & $3.84 \pm 0.00$ & $100.0\% \pm 0.00\%$ \\ \hline
    \end{tabular}}
    \caption{Average UFM feature loss and accuracy over 10 trials with and without normalization to sphere, with the same set-up as in \Cref{fig:low_dim_features}.}
    \label{tab:low_dim_acc}
    \renewcommand{\arraystretch}{1}
\end{table}

\paragraph{Motivations \& contributions.}

In this work, we further demystify why \NC\ happens in network training with a common practice of feature normalization (i.e., normalizing the last-layer features on the unit hypersphere), mainly motivated by the following reasons:
\begin{itemize}[leftmargin=*]
    \item \emph{Feature normalization is a common practice in training deep networks.} 
    Recently, many existing results demonstrated that training with feature normalization often improves the quality of learned representation with better class separation \cite{ranjan2017l2,liu2017sphereface,wang2018cosface,deng2019arcface,chan2020deep,yu2020learning,wang2020understanding,graf2021dissecting}. Such a representation is closely related to the discriminative representation in literature; see, e.g., \cite{ranjan2017l2,liu2017sphereface,wang2018cosface,yu2022towards}. As illustrated in \Cref{fig:low_dim_features} and \Cref{tab:low_dim_acc}, experimental results visualized in low-dimensional space show that features learned with normalization are more uniformly distributed over the sphere and hence are more linearly separable than those learned without normalization. In particular, it has been shown that the learned representations with larger class separation usually lead to improved test performances; see, e.g., \cite{khosla2020supervised,graf2021dissecting}. Moreover, it has been demonstrated that discriminative representations can also improve robustness to mislabeled data \cite{chan2020deep,yu2020learning}, and has become a common practice in recent advances on (self-supervised) pretrained models \cite{chen2020simple,wang2020understanding}.
    \item \emph{A common practice of theoretically studying \NC\ with norm constraints.} Due to these practical reasons, many existing theoretical studies on \NC\ consider formulations with both the norms of features and classifiers constrained \cite{lu2020neural,wojtowytsch2020emergence,fang2021exploring,graf2021dissecting,ji2021unconstrained}. Based upon assumptions of unconstrained feature models \cite{mixon2020neural,zhu2021geometric,fang2021exploring}, these works show that the only global solutions satisfy \NC\ properties for a variety of loss functions (e.g., MSE, CE, SC losses, etc). Nonetheless, they only focused on the global optimality conditions without looking into the nonconvex landscapes, and therefore failed to explain why these \NC\;solutions can be efficiently reached by classical training algorithms such as stochastic gradient descent (SGD).
\end{itemize}
In this work we study the global nonconvex landscape of training deep networks with norm constraints on the features and classifiers. We consider the commonly used CE loss and formulate the problem as a Riemannian optimization problem over products of unit spheres (i.e., the oblique manifold). Our study is also based upon the assumption of the so-called \emph{unconstrained feature model} (UFM) \cite{mixon2020neural,zhu2021geometric,zhou2022optimization} or \emph{layer-peeled model} \cite{fang2021layer}, where the last-layer features of the deep network are treated as free optimization variables to simplify the nonlinear interactions across layers. The underlying reasoning is that modern deep networks are often highly overparameterized with the capacity of learning any representations \cite{lu2017expressive,hornik1991approximation,shaham2018provable}, so that the last-layer features can approximate, or interpolate, any point in the feature space. \\~\\
Assuming the UFM, we show that the Riemannian optimization problem has a benign global landscape, in the sense that the loss with respect to (w.r.t.) the features and classifers is a strict saddle function \cite{ge2015escaping,sun2015nonconvex} over the Riemannian manifold. More specifically, we prove that \emph{every} local minimizer is a global solution satisfying the \NC\ properties, and \emph{all} the other critical points exhibit directions with negative curvature. Our analysis for the manifold setting is based upon a nontrivial extension of recent studies for the \NC\ with penalized  formulations \cite{zhu2021geometric,zhou2022optimization,tirer2022extended,han2021neural}, which could be of independent interest. Our work brings new tools from Riemannian optimization for analyzing optimization landscapes of training deep networks with an increasingly common practice of feature normalization. At the same time, we empirically demonstrate the advantages of the Riemannian/constrained formulation over its penalized counterpart for training deep networks -- faster training and higher quality representations. Lastly, under the UFM we believe that the benign landscape over the manifold could hold for many other popular training losses beyond CE, such as the (supervised) contrastive loss \cite{khosla2020supervised}. We leave this for future exploration.

\paragraph{Prior arts and related works on \NC.} The empirical \NC\ phenomenon has inspired a recent line of theoretical studies on understanding why it occurs ~\cite{graf2021dissecting,fang2021layer,lu2020neural,mixon2020neural,tirer2022extended,zhou2022optimization,zhu2021geometric}. Like ours, most of these works studied the problem under the UFM. In particular, despite the nonconvexity, recent works showed that the \emph{only} global solutions are \NC\ solutions for a variety of nonconvex training losses (e.g., CE \cite{zhu2021geometric,fang2021layer,lu2020neural}, MSE \cite{tirer2022extended,zhou2022optimization}, SC losses \cite{graf2021dissecting}) and different problem formulations (e.g., penalized, constrained, and unconstrained) \cite{zhu2021geometric,zhou2022optimization,tirer2022extended,han2021neural,lu2020neural}. Recently, this study has been extended to deeper models with the MSE training loss \cite{tirer2022extended}. More surprisingly, it has been further shown that the nonconvex losses under the UFM have benign global optimization landscapes, in the sense that every local minimizer satisfies \NC\ properties and the remaining critical points are strict saddles with negative curvature. Such results have been established for both CE and MSE losses \cite{zhu2021geometric,zhou2022optimization}, where they considered the unconstrained formulations with regularization on both features and classifiers. We should also mention that the benign global optimization landscapes of many other problems in neural networks have been widely found in the literature; see, e.g., \cite{zhang2021expressivity,laurent2018deep,sun2020global,liu2022loss,soltanolkotabi2018theoretical}.  

Moreover, there is a line of recent works investigating the benefits of \NC\ on generalizations of deep networks. The work \cite{galanti2022on} shows that \NC\ also happens on test data drawn from the same distribution asymptotically, but less collapse for finite samples \cite{hui2022limitations}. Other works \cite{hui2022limitations,papyan2020traces} demonstrated that the variability collapse of features is actually happening progressively from shallow to deep layers, and \cite{ben2022nearest} showed that test performance can be improved when enforcing variability collapse on features of intermediate layers. The works \cite{xie2022neural,yang2022we} showed that fixing the classifier as a simplex ETF improves test performance on imbalanced training data and long-tailed classification problems. We refer interested readers to a recent survey on this emerging topic \cite{kothapalli2022neural}.

\paragraph{Notation.} Let $\R^n$ be the $n$-dimensional Euclidean space and $\|\cdot\|_2$ be the Euclidean norm. We write matrices in bold capital letters such as $\mb A$, vectors in bold lower-case such as $\mb a$, and scalars in plain letters such as $a$. Given a matrix $\mb A \in \R^{d\times K}$, we denote its $k$-th column by $\mb a_k$, its $i$-th row by $\mb a^i$, its $(i,j)$-th element by $a_{ij}$, and let $\|\mb A\|$ be its spectral norm. We use $\diag(\mb A)$ to denote a vector that consists of diagonal elements of $\mb A$, and we use $\ddiag(\mb A)$ to denote a diagonal matrix composed by only the diagonal entries of $\mb A$. Given a positive integer $n$, we denote the set $\{1,\dots,n\}$ by $[n]$. We denote the unit hypersphere in $\R^d$ by $\bb S^{d-1} := \{\mb x \in \R^d: \|\mb x\|_2 = 1\}$. 






\section{Nonconvex Formulation with Spherical Constraints}\label{sec:formulation}
In this section, we review the basic concepts of deep neural networks and introduce notation that will be used throughout the paper. Based upon this, we formally introduce the problem formulation over the Riemannian manifold under the assumption of the UFM.

\subsection{Basics of Deep Neural Networks}\label{subsec:basics}

In this work, we focus on the multi-class (e.g., $K$ classes) classification problem. Given input data $\mb x \in \bb R^D$, the goal of deep learning is to learn a deep hierarchical representation (or feature) $\mb h(\mb x)=\phi_{\mb \theta}(\mb x) \in \bb R^d $ of the input along with a linear classifier\footnote{We write $\mb W = \mb W_L^\top$ in the transposed form for the simplicity of analysis.} $\mb W \in \bb R^{d \times K}$ such that the output $\psi_{\mb \Theta}(\mb x) = \mb W^\top \mb h(\mb x)$ of the network fits the input $\mb x$ to an one-hot training label $\mb y\in \bb R^K$.
More precisely, in vanilla form an $L$-layer fully connected deep neural network can be written as
\begin{align}\label{eq:func-NN}
    \psi_{\mb \Theta}(\mb x) \;=\;  \underbrace{ \mb W_L}_{ \text{\bf linear classifier}\;\mb W = \mb W_L^\top } \;   \; \underbrace{\sigma\paren{ \mb W_{L-1} \cdots \sigma \paren{\mb W_1 \mb x + \mb b_1} + \mb b_{L-1} }}_{\text{ \bf feature}\;\; \mb h\;=\; \phi_{\mb \theta}(\mb x)}  \;+\; \mb b_L,
\end{align}
where each layer is composed of an affine transformation, represented by some weight matrix $\mb W_k$, and bias $\mb b_k$, followed by a nonlinear activation $\sigma(\cdot)$, and $\mb \Theta = \Brac{ \mb W_k, \mb b_k }_{k=1}^L$ and $\mb \theta = \Brac{ \mb W_k, \mb b_k }_{k=1}^{L-1}$ denote the weights for \emph{all} the network parameters and those up to the last layer, respectively. 
Given training samples $\{(\mb x_{k,i},\mb y_k)\} \subseteq \R^D \times \R^K$ drawn from the same data distribution $\mc D$, we learn the network parameters $\mb \Theta$ via minimizing the empirical risk over these samples,
\begin{align}\label{eqn:dl-ce-loss}
    \min_{\mb \Theta} \; \sum_{k=1}^K \sum_{i=1}^{n_k} \mc L_{\mathrm{CE}} \paren{ \psi_{\mb \Theta}(\mb x_{k,i}),\vy_k }, \quad \text{s.t.}\quad \mb \Theta \in \mc C,
\end{align}
where $\vy_k \in \bb R^K $ is a one-hot vector with only the $k$-th entry being $1$ and the remaining ones being $0$ for all $k \in [K]$, $\mb x_{k,i} \in \bb R^D$ is the $i$-th sample in the $k$-th class, $\Brac{n_k}_{k=1}^K$ denotes the number of training samples in each class, and the set $\mc C$ denotes the constraint set of the network parameters $\mb \Theta$ that we will specify later. In this work, we study the most widely used CE loss of the form
\begin{align*}
    \mc L_{\mathrm{CE}}(\mb z, \mb y_k) \;:=\; - \log \paren{  \frac{ \exp(z_k) }{ \sum_{\ell=1}^K \exp(z_{\ell}) } }.
\end{align*}

\subsection{Riemannian Optimization over the Product of Spheres}
        
For the $K$-class classification problem, let us consider a simple case where the number of training samples in each class is balanced (i.e., $n=n_1=n_2= \cdots =n_K$) and let $N=Kn$, and we assume that all the biases $\Brac{\mb b_k}_{k=1}^L$ are zero with the last activation function $\sigma(\cdot)$ before the output to be linear. Analyzing deep networks $\psi_{\mb \Theta}(\mb x)$ is a tremendously difficult task mainly due to the \emph{nonlinear interactions} across a large number of layers. 
To simplify the analysis, we assume the so-called \emph{unconstrained feature model} (UFM) following the previous works \cite{ji2021unconstrained,mixon2020neural,graf2021dissecting,zhu2021geometric}. More specifically, we simplify the nonlinear interactions across layers by treating the last-layer features $\mb h_{k,i} = \phi_{\mb \theta}(\mb x_{k,i}) \in \bb R^d$ as \emph{free} optimization variables,
where the underlying reasoning is that modern deep networks are often highly overparameterized to approximate any continuous function \cite{lu2017expressive,hornik1991approximation, shaham2018provable}. Concisely, we rewrite all the features in a matrix form as 
\begin{align*}
    \mb H \;=\; \begin{bmatrix}
    \mb H_1 & \mb H_2 & \cdots & \mb H_K
    \end{bmatrix} \in \bb R^{d \times N}, \; \mb H_k \;=\; \begin{bmatrix}
    \mb h_{k,1} & \mb h_{k,2} & \cdots & \mb h_{k,n}
    \end{bmatrix} \in \bb R^{d \times n}, \;\forall\ k \in [K],
\end{align*}
and correspondingly denote the classifier $\mb W$ by
\begin{align*}
    \mb W \;=\; \begin{bmatrix}
    \mb w_1 & \mb w_2 & \cdots & \mb w_K 
    \end{bmatrix} \in \bb R^{ d \times K},\quad \mb w_k \in \bb R^d,\quad \forall\ k \in [K].
\end{align*}Based upon the discussion in \Cref{sec:intro}, we assume that both the features $\mb H$ and the classifiers $\mb W$ are normalized,\footnote{In practice, it is a common practice to normalize the output feature $\mb h_o$ by its norm, i.e., $\mb h = \mb h_o /\norm{\mb h_o}{2}$, so that $\norm{\mb h}{2}=1$.} in the sense that $\|\mb h_{k,i} \|_2 = 1$, $\|\mb w_k\|_2 = \tau$, for all $k \in [K]$ and all $i \in [n]$, where $\tau>0$ is a temperature parameter. 
As a result, we obtain a \emph{constrained} formulation over a Riemannian manifold 
\begin{align}\label{P:1}
     \min_{\mb W , \mb H } \; \frac{1}{N} \sum_{k=1}^K \sum_{i=1}^{n} \mc L_{\mathrm{CE}} \paren{  \mb W^\top \mb h_{k,i}, \mb y_k } \quad \text{s.t.}\; \|\mb w_k\|_2 \;=\; \tau,\; \|\mb h_{k,i}\|_2 \;=\;1,\ \forall\ i \in [n],\;\forall\ k \in [K].
\end{align}
Since the temperature parameter $\tau$ can be absorbed into the loss function, we replace $\mb w_k$ by $\tau \mb w_k$ and change the original constraint into $\|\mb w_k\|_2 = 1$ for all $k \in [K]$. 
In particular, the product of spherical constraints forms an \emph{oblique manifold} \cite{boumal2022intromanifolds} embedded in Euclidean space, 
\begin{align*}
    \mc {OB}(d,K) \;:=\; \Brac{ \mb Z \in \bb R^{d \times K} \mid \mb z_k \in \bb S^{d-1},\;\forall\ k \in [K] }.
\end{align*}
Consequently, we can rewrite Problem \eqref{P:1} as a Riemannian optimization problem over the oblique manifold w.r.t. $\mb W$ and $\mb H$:
\begin{align} 
     \min_{\mb W, \mb H} &\; f(\mb W,\mb H) \;:=\; \frac{1}{N} \sum_{k=1}^K \sum_{i=1}^{n} \mc L_{\mathrm{CE}} \paren{ \tau \mb W^\top \mb h_{k,i}, \mb y_k }, \label{eq:ce-loss func} \\
     \;\text{s.t.} & \quad \mb H \in \mc {OB}(d,N),\;  \mb W \in\; \mc {OB}(d,K). \nonumber 
\end{align}
In \Cref{sec:results}, we will show that all global solutions of Problem \eqref{eq:ce-loss func} satisfy \NC\ properties, and its objective function  is a strict saddle function \cite{jin2017escape,du2017gradient} of $(\mb W,\mb H)$ over the oblique manifold so that the \NC\ solution can be efficiently achieved. 

\begin{wrapfigure}{R}{0.37\textwidth}
\centering	\includegraphics[height=1.8in]{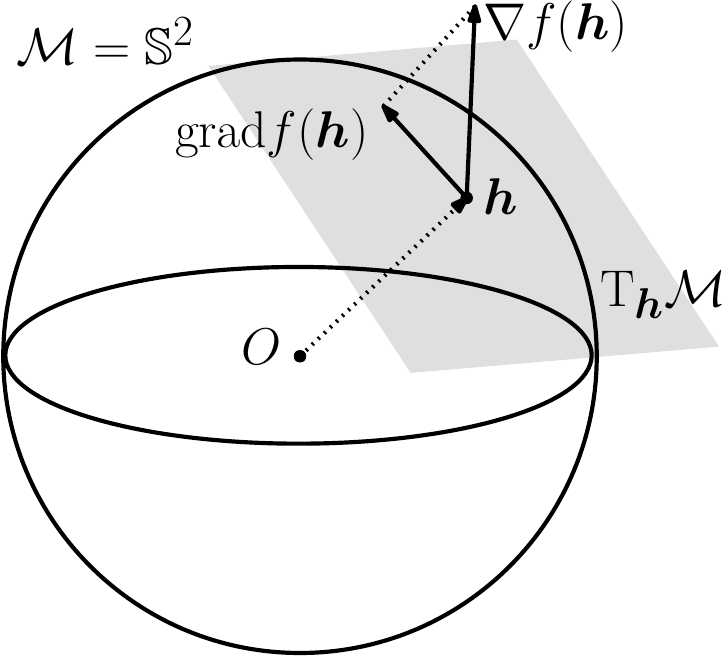}
	\caption{\textbf{An illustration of the Riemannian gradient of $f(\mb h)$ on a simple manifold $\mc {OB}(3,1)$.}
	}\label{fig:manifold}
\end{wrapfigure}
\paragraph{Riemannian derivatives over the oblique manifold.} In \Cref{sec:results}, we will use tools from Riemannian optimization to characterize the global optimality condition and the geometric properties of the optimization landscape of Problem \eqref{eq:ce-loss func}. Before that, let us first briefly introduce some basic derivations of the Riemannian gradient and Hessian, defined on the tangent space of the oblique manifold. For more technical details, we refer the readers to \Cref{subsec:Rie tools}. 
According to \cite[Chapter 3 \& 5]{boumal2022intromanifolds} and \cite{hu2020brief,absil2009optimization}, we can calculate the Riemannian gradients and Hessian of Problem \eqref{eq:ce-loss func} as follows. Since those quantities are defined on the tangent space, according to \cite[Section 3.1]{hu2020brief} and the illustration in \Cref{fig:manifold}, we first obtain the tangent space to $\mc {OB}(d,K)$ at $\mb W$ as
\begin{align*}
    \mathrm{T}_{\mb W} \mc {OB} (d,K)  \;=\; \Brac{ \mb Z \in \bb R^{d \times K} \mid \diag \paren{ \mb W^\top \mb Z } = \mb 0 }.
\end{align*}
Note that the tangent space contains all $\mb Z$ such that $\mb z_k$ is orthogonal to $\mb w_k$ for all $k$. When $K = 1$, it reduces to the tangent space to the unit sphere $\bb S^{d-1}$. \\
~\\
Analogously, we can derive the tangent space for $\mb H$ with a similar form. Let us define
\begin{align*}
    \mb M := \tau\mb W^\top\mb H,\ g(\mb M) := f(\mb W, \mb H). 
\end{align*}
First, the Riemannian gradient of $f(\mb W, \mb H)$ of Problem \eqref{eq:ce-loss func} is basically the projection of the ordinary Euclidean gradient $\nabla f(\mb W, \mb H)$ onto its tangent space, i.e., $\grad_{\mb W} f(\mb W,\mb H) = \mc P_{\mathrm{T}_{\mb W} \mc {OB} (d,K)}( \nabla_{\mb W} f(\mb W,\mb H) ) $ and $\grad_{\mb H} f(\mb W,\mb H) = \mc P_{\mathrm{T}_{\mb H} \mc {OB} (d,N)}( \nabla_{\mb H} f(\mb W,\mb H) ) $. More specifically, we have 
\begin{align}
        \grad_{\mb W } f(\mb W,\mb H) \;&=\; \tau \mb H \nabla g(\mb M)^\top  - \tau \mb W \ddiag \paren{ \mb W^\top \mb H \nabla g(\mb W)^\top }\label{eq:rigrad W}, \\ 
        \grad_{\mb H} f(\mb W,\mb H) \;&=\;\tau \mb W \nabla g(\mb M) - \tau\mb H  \ddiag \paren{ \mb H^\top \mb W \nabla g(\mb M)  }. \label{eq:rigrad H}
\end{align}
Second, for any $\mb \Delta = (\mb \Delta_{\mb W} , \mb \Delta_{\mb H} ) \in \bb R^{d \times K} \times \bb R^{d \times N}$, we compute the Hessian bilinear form of $f(\mb W,\mb H)$ along the direction $\mb \Delta$ by
\begin{align}\label{eq:euc-hessian}
    \nabla^2 f(\mb W,\mb H)[ \mb \Delta,\mb \Delta ] \;=\;&  \nabla^2g(\mb M) \brac{ \tau\left(\mW^\top \mDelta_{\mH} + \mDelta_{\mW}^\top \mH \right), \tau\left(\mW^\top \mDelta_{\mH} + \mDelta_{\mW}^\top \mH\right) } \nonumber \\
    &+ 2\tau \innerprod{\nabla g(\mb M)}{\mDelta_{\mW}^\top \mDelta_{\mH}}.
\end{align}
We compute the Riemannian Hessian bilinear form of $f(\mb W,\mb H)$ along any direction $\mb \Delta \in \mathrm{T}_{ \mb W }  \mc {OB}(d,K )  \times \mathrm{T}_{ \mb H}\mc {OB} (d,N) $ by
\begin{align}\label{eq:bilinear Hess}
    \Hess f(\mb W, \mb H)[\mb \Delta,\mb \Delta] \;=\;&  \nabla^2 f(\mb W,\mb H)[ \mb \Delta,\mb \Delta ] - \langle \mb \Delta_{\mb W}\ddiag\left(\mb M\nabla g(\mb M)^\top\right), \mb \Delta_{\mb W}  \rangle \notag\\
    & - \langle \mb \Delta_{\mb H}\ddiag\left(\mb M^\top\nabla g(\mb M)\right), \mb \Delta_{\mb H}  \rangle,
\end{align}
where the extra terms besides $\nabla^2 f(\mb W,\mb H)[ \mb \Delta,\mb \Delta ]$ represent the curvatures induced by the oblique manifold. We refer to \Cref{app:hessderivation} for the derivations of \eqref{eq:euc-hessian} and \eqref{eq:bilinear Hess}.
In the following section, we will use the Riemannian gradient and Hessian to characterize the optimization landscape of Problem \eqref{eq:ce-loss func}.

\section{Main Theoretical Analysis}\label{sec:results}
In this section,  we first characterize the structure of the global solution set of Problem \eqref{eq:ce-loss func}. Based upon this, we analyze the global landscape of Problem \eqref{eq:ce-loss func} via characterizing its Riemannian derivatives.

\subsection{Global Optimality Condition}

For the feature matrix $\mb H$, let us denote the class mean for each class by
\begin{align}\label{eqn:H-mean}
    \ol{\mb h}_k\;:=\; \frac{1}{n} \sum_{i=1}^n \mb h_{k,i},\ \forall\ k \in [K],\quad \text{and}\quad \ol{\mb H} \;:=\; \begin{bmatrix}
     \ol{\mb h}_1 & \cdots & \ol{\mb h}_K
     \end{bmatrix} \in \bb R^{d \times K}.
\end{align}
Based upon this, we show any global solution of Problem \eqref{eq:ce-loss func} exhibits \NC\ properties in the sense that it satisfies ({\bf NC1}) variability collapse, ({\bf NC2}) convergence to simplex ETF, and ({\bf NC3}) convergence to self-duality.

\begin{theorem}[Global Optimality Condition]\label{thm:ce optim}
Suppose that the feature dimension is no smaller than the number of classes, i.e., $d \ge K$, and the training labels are balanced in each class, i.e.,  $n = n_1 = \cdots =n_K$. 
Then for the CE loss $f(\mb W, \mb H)$ in Problem \eqref{eq:ce-loss func}, it holds that
\begin{align*}
             f(\mb W,\mb H) \, &\;\geq \; \log \left( 1+(K-1)\exp\left(-\frac{K\tau}{K-1}\right) \right)
\end{align*}
for all $\mb W = [\mb w_1, \dots, \mb w_K] \in \mc {OB}(d,K)$ and $\mb H = [\mb h_{1,1},\dots, \mb h_{K,n}] \in \mc {OB}(d,N)$. In particular, equality holds if and only if 
    \begin{itemize}[leftmargin=*]
        \item{({\bf NC1})} \textbf{Variability collapse:} $\mb h_{k,i} = \ol{\mb h}_k,\ \forall\ i\in [n]$;  
        \item{({\bf NC2})} \textbf{Convergence to simplex ETF:} $\{ \ol{\mb h}_k \}_{k=1}^{K}$ form a sphere-inscribed simplex ETF in the sense that 
        \begin{align*}
            \ol{\mb H}^\top \ol{\mb H}  \;=\;\frac{1}{K-1}  \paren{ K \mb I_K - \mb 1_K \mb 1_K^\top },\quad \ol{\mb H} \in \mc {OB}(d,K).
        \end{align*}
        \item{({\bf NC3})} \textbf{Convergence to self-duality:} $\mb w_k = \ol{\mb h}_k,\ \forall\ k \in [K]$.
    \end{itemize}
\end{theorem}
Compared to the unconstrained regularized problems in \cite{zhou2022optimization, zhu2021geometric}, it is worth noting that the regularization parameters influence the structure of global solutions, while the temperature parameter $\tau$ only affects the optimization landscape but not the global solutions. On the other hand, our result is closely related to \cite[Theorem 1]{graf2021dissecting} (i.e., spherical constraints vs. that of ball constraints). In fact, our problem and that in \cite{graf2021dissecting} share the same global solution set. Moreover, our proof follows similar ideas as those in a line of recent works \cite{graf2021dissecting,fang2021layer,lu2020neural,tirer2022extended,zhou2022optimization,zhu2021geometric}, and we refer the readers to \Cref{app:thm global} for the proof. It should be noted that we do not claim originality of this result compared to previous works; instead our major contribution lies in the following global landscape analysis.

\begin{figure}[t]
    \centering
    \includegraphics[width=0.9\textwidth]{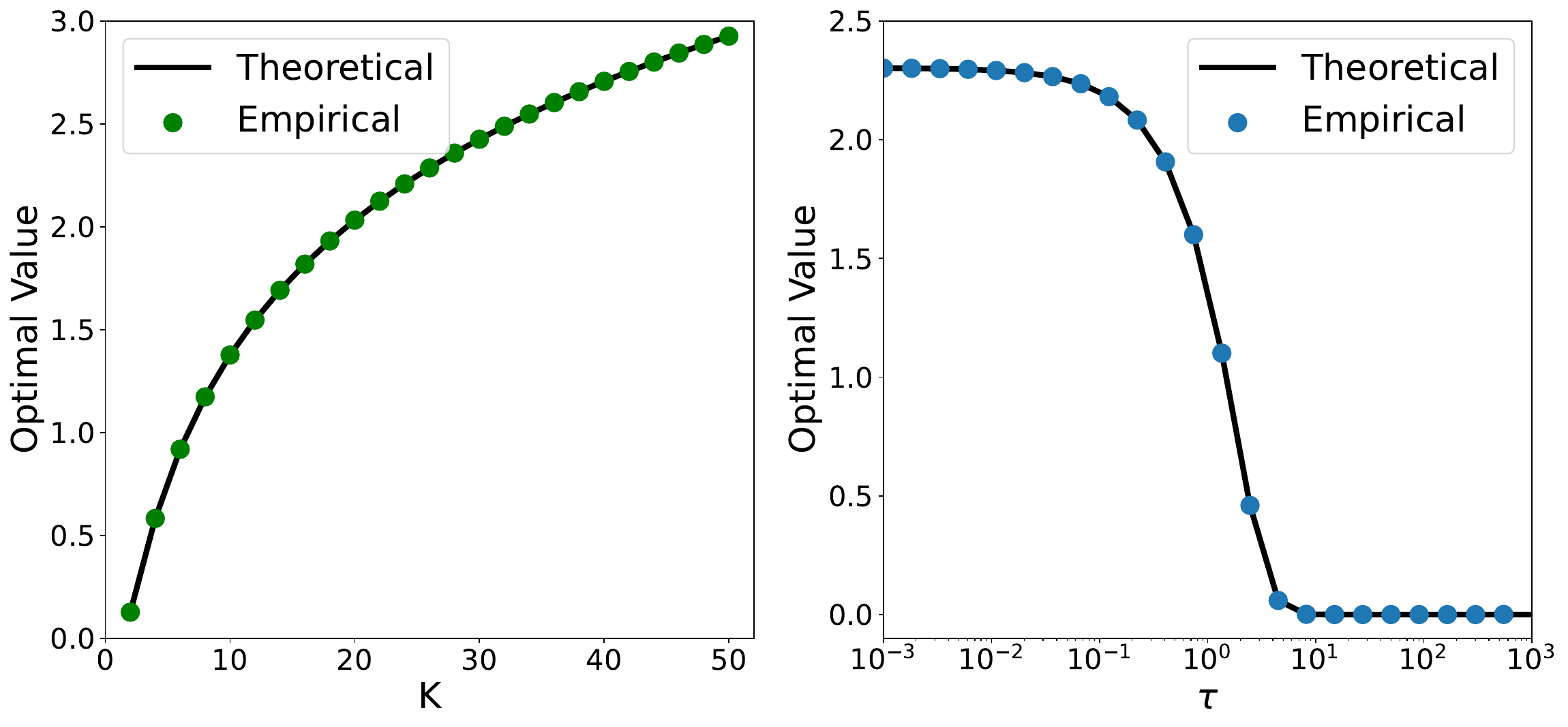}
    \caption{\textbf{Global optimization of \eqref{eq:ce-loss func} under UFM with $d=100$ and $n=5$.} Theoretical line refers to lower bound (global minimum) from \Cref{thm:ce optim}. Empirical values found using gradient descent with random initialization. Left: Lower bound against number of classes $K$ while fixing $\tau=1$. Right: Lower bound against temperature $\tau$ while fixing $K=10$. The same empirical values are achieved over many trials due to the benign global landscape.}
    \label{fig:lower_bound}
\end{figure}

\subsection{Global Landscape Analysis}

Due to the nonconvex nature of Problem \eqref{eq:ce-loss func}, the characterization of global optimality in \Cref{thm:ce optim} alone is not sufficient for guaranteeing efficient optimization to those desired global solutions. Thus, we further study the global landscape of Problem \eqref{eq:ce-loss func} by characterizing all the \emph{Riemannian critical points} $(\mb W,\mb H) \in \mc {OB}(d,K) \times \mc {OB}(d,N)$ satisfying
\begin{align*}
    \grad_{\mb H} f(\mb W,\mb H) \;=\; \mb 0,\quad \grad_{\mb W} f(\mb W,\mb H) \;=\; \mb 0.
\end{align*}
We now state our major result below.
\begin{theorem}[Global Landscape Analysis]\label{thm:ce landscape}
Assume that the number of training samples in each class is balanced, i.e., $n = n_1 = \cdots = n_K$. If the feature dimension is larger than the number of classes, i.e., $d > K$, and the temperature parameter satisfies $\tau < 2(d-2)(1+(\Kmod)/K)^{-1}$, then the function $f(\mb W,\mb H)$ is a strict saddle function that has no spurious local minimum, in the sense that 
\begin{itemize}[leftmargin=*]
    \item Any Riemannian critical point $(\mb W, \mb H)$ of Problem \eqref{eq:ce-loss func} that is not a local minimizer is a Riemannian strict saddle point with negative curvatures, in the sense that the Riemannian Hessian $\mathrm{Hess} f(\mb W, \mb H)$ at the critical point $(\mb W,\mb H)$ is non-degenerate, and there exists a direction $\mb \Delta = (\mb \Delta_{\mb W},\mb \Delta_{\mb H}) \in \mathrm{T}_{ \mb W }  \mc {OB}(d,K )  \times \mathrm{T}_{ \mb H}\mc {OB} (d,N) $ such that
    \begin{align*}
        \Hess f(\mb W, \mb H)[\mb \Delta,\mb \Delta] \;<\; 0.
    \end{align*}
    In other words, $\lambda_{\mathrm{min}}\left(\mathrm{Hess} f(\mb W, \mb H)\right) < 0$ at the corresponding Riemannian critical point.
    \item Any local minimizer of Problem \eqref{eq:ce-loss func} is a global minimizer of the form shown in \Cref{thm:ce optim}. 
\end{itemize}
\end{theorem}
For the details of the proof, we refer readers to \Cref{app:thm landscape}.
The second bullet point naturally follows from \Cref{thm:ce optim} and the first bullet point. The major challenge of our analysis is showing the first bullet, i.e., to find a negative curvature direction $\mb \Delta$ for $\Hess f(\mb W, \mb H)$. Our key observation is that the set of non-global critical points can be partitioned into two separate cases. In the first case, the last two terms of \eqref{eq:bilinear Hess} vanish, and we show that the second term of \eqref{eq:euc-hessian} is negative and dominates the first term for an appropriate direction. We require $\tau$ to not be too large, since the first term is $O(\tau^2)$, whereas the second term is $O(\tau)$. In the second case, using the assumption that $d > K$ we can find a rank-one direction that makes the first term of \eqref{eq:euc-hessian} vanishing. In this case, we similarly show that the second term of \eqref{eq:euc-hessian} is negative but instead dominates the last two terms of \eqref{eq:bilinear Hess}. In the following, we discuss the implications, relationship, and limitations of our results in \Cref{thm:ce landscape}.
\begin{itemize}[leftmargin=*]
    \item \emph{Efficient global optimization to \NC\ solutions.} Our theorem implies that the \NC\ solutions can be efficiently reached by Riemannian first-order methods (e.g., Riemannian stochastic gradient descent) with random initialization \cite{jin2017escape,criscitiello2019efficiently,sun2019escaping} for solving Problem \eqref{eq:ce-loss func}; see \Cref{fig:lower_bound} for a demonstration. For training practical deep networks, this can be efficiently implemented by normalizing last-layer features when running SGD. 
    \item \emph{Relation to existing works on \NC.} Most existing results have only studied the global minimizers under the UFM \cite{graf2021dissecting,fang2021layer,lu2020neural,tirer2022extended}, which has limited implication for optimization. On the other hand, our landscape analysis is based upon a nontrivial extension of that with the unconstrained problem formulation \cite{zhu2021geometric,zhou2022optimization}. Compared to those works, Problem \eqref{eq:ce-loss func} is much more challenging for analysis, due to the fact that the set of critical points of our problem is essentially much larger than that of \cite{zhu2021geometric,zhou2022optimization}. Moreover, we empirically demonstrate the advantages of the manifold formulation over its regularized counterpart, in terms of representation quality and training speed.
    \item \emph{Assumptions on the feature dimension $d$ and temperature parameter $\tau$.} Our current result requires that $d > K$, which is the same requirement in \cite{zhu2021geometric,zhou2022optimization}. Furthermore, through numerical simulations we conjecture that the global landscape also holds even when $d\ll K$, while the global solutions are uniform over the sphere \cite{lu2020neural} rather than being simplex ETFs (see \Cref{fig:low_dim_features}). The analysis on $d\ll K$ is left for future work.
    On the other hand, the required upper bound on $\tau$ is for the ease of analysis and it holds generally in practice,\footnote{For instance, a standard ResNet-18 \cite{he2016deep} model trained on CIFAR-10 \cite{krizhevsky2009learning} has $d=512$ and $K=10$. In the same setting, we assume $\tau < 1020$, which is far larger than any useful setting of the temperature parameter (see \Cref{sec:tau}).} but we conjecture that the benign landscape holds without it.

    \item \emph{Relation to other Riemannian nonconvex problems.} Our result joins a recent line of work on the study of global nonconvex landscapes over Riemannian manifolds, such as orthogonal tensor decomposition \cite{ge2015escaping}, dictionary learning \cite{qu2014finding,sun2016complete1,sun2015complete,qu2020geometric,qu2020finding}, subspace clustering \cite{wang2022convergence}, and sparse blind deconvolution \cite{kuo2019geometry,qu2020exact,lau2020short,zhang2019structured}. For all these problems constrained over a Riemannian manifold, it can be shown that they exhibit ``equivalently good'' global minimizers due to symmetries and intrinsic low-dimensional structures, and the loss functions are usually strict saddles \cite{ge2015escaping,sun2015nonconvex,zhang2020symmetry}. As we can see, the global minimizers (i.e., simplex ETFs) of our problem here also exhibit a similar rotational symmetry, in the sense that $\mb W^\top \mb H = (\mb Q\mb W)^\top \paren{ \mb Q \mb H}$ for any orthogonal matrix $\mb Q$. Additionally, our result show that the tools from Riemmanian optimization can be powerful for the study of deep learning.
\end{itemize}

\section{Experiments}\label{sec:experiments}
In this section, we support our theoretical results in previous sections and provide further motivation with experimental results on practical deep network training. In \Cref{sec:validation_exp}, we validate the assumption of UFM introduced in \Cref{sec:formulation} for analyzing \NC, by demonstrating that \NC\ occurs for increasingly overparameterized deep networks. In \Cref{sec:fast_train}, we further motivate feature normalization with empirical results showing that feature normalization can lead to faster training and better collapse than the unconstrained counterpart with regularization. This occurs not only with the UFM but also with practical overparameterized networks. In \Cref{sec:cg_trm}, we demonstrate that, independent of the algorithm used, the feature normalized UFM has faster training and collapse than its regularized counterpart. In \Cref{sec:generalization}, we show that feature normalization leads to better generalization and test feature collapse on practical deep networks. In \Cref{sec:tau}, we investigate the effect of the temperature parameter $\tau$ on training dynamics for both the UFM and practical deep networks. Finally, in \Cref{sec:other-losses}, we empirically explore the global landscape of other commonly used loss functions for deep learning classification tasks. Before that, we introduce some basics of the experimental setup and metrics for evaluating \NC.

\paragraph{Network architectures, datasets, and training details.} In our experiments, we use ResNet \cite{he2016deep} architectures for the feature encoder. For the normalized network, we project the output of the encoder onto the sphere of radius $\tau$ (as done in \cite{graf2021dissecting}) and also project the weight classifiers to the unit sphere after each optimization step to maintain constraints. In all experiments, we set $\tau=1$. For the regularized UFM and network, we use a weight and feature decay of $10^{-4}$ (using the loss in \cite{zhu2021geometric}). We do not use a bias term for the classifier for either architecture. For all experiments, we use the CIFAR dataset\footnote{Both CIFAR10 and CIFAR100 are publicly available and are licensed under the MIT license.} \cite{krizhevsky2009learning}, where we use CIFAR100 for all experiments except for the experiment in \Cref{sec:validation_exp}, where we use CIFAR10. In all experiments, we train the networks using SGD with a batch size of $128$ and momentum $0.9$ with an initial learning rate of $0.05$, and we decay the learning rate by a factor of $0.1$ after every $40$ epochs - these hyperparameters are chosen to be the same as those in \cite{zhu2021geometric} for fair comparisons. All networks are trained on Nvidia Tesla V100 GPUs with 16G of memory.

\paragraph{Neural collapse metrics.} For measuring different aspects of neural collapse as introduced in \Cref{sec:intro}, we adopt similar \NC\ metrics from \cite{papyan2020prevalence,zhu2021geometric,zhou2022optimization}, given by
\begin{align*}
    \mathcal{NC}_1 &:= \frac{1}{K}\mbox{trace}(\mb \Sigma_W \mb \Sigma_B^{\dagger}), \\
    \mathcal{NC}_2 &:= \left\|\frac{\mb W^\top \mb W}{\|\mb W^\top \mb W\|_F} - \frac{1}{\sqrt{K-1}}\left(\mb I_K - \mb 1_K \mb 1_K^\top\right)\right\|_F, \\
    \mathcal{NC}_3 &:= \left\|\frac{\mb W^\top \overline{\mb H}}{\|\mb W^\top \overline{\mb H}\|_F } - \frac{1}{\sqrt{K-1}}\left(\mb I_K - \mb 1_K \mb 1_K^\top\right)\right\|_F,
\end{align*}
where $\mb \Sigma_W$ and $\mb \Sigma_B$ are the within-class and between-class covariance matrices (see \cite{papyan2020prevalence,zhu2021geometric} for more details), $\mb \Sigma_B^\dagger$ denotes pseudo inverse of $\mb \Sigma_B$, and $\overline{\mb H}$ is the centered class mean matrix in \eqref{eqn:H-mean}. More specifically, $\mc {NC}_1$ measures {\bf NC1} (i.e., within class variability collapse), $\mc {NC}_2$ measures {\bf NC2} (i.e., the convergence to the simplex ETF), and $\mc {NC}_3$ measures {\bf NC3} (i.e., the duality collapse).

\subsection{Validation of the UFM for training networks with feature normalization}\label{sec:validation_exp}

In \Cref{sec:formulation}, our study of the Riemannian optimization problem \eqref{eq:ce-loss func} is based upon the UFM, where we assume $\mb H$ is a free optimization variable. Here, we justify this assumption by showing that \NC\ happens for training overparameterized networks even when the training labels are completely random. By using random labels, we disassociate the input from their class labels, by which we can characterize the approximation power of the features of overparmeterized models. To show this, we train ResNet-18 with varying widths (i.e., the number of feature maps resulting from the first convolutional layer) on CIFAR10 with random labels, with normalized features and classifiers. \\~\\
As shown in \Cref{fig:nc_width}, we observe that increasing the width of the network allows for perfect classification on the training data even when the labels are random. Furthermore, increasing the network width also leads to better \NC, measured by the decrease in each $\mathcal{NC}$ metric. This corroborates that (\emph{i}) our assumption of UFM is reasonable given that \NC\ seems to be independent of the input data, and (\emph{ii}) \NC\ happens under the our constraint formulation \eqref{eq:ce-loss func} on practical networks.

\begin{figure}[t]
    \centering
    \includegraphics[width=\textwidth]{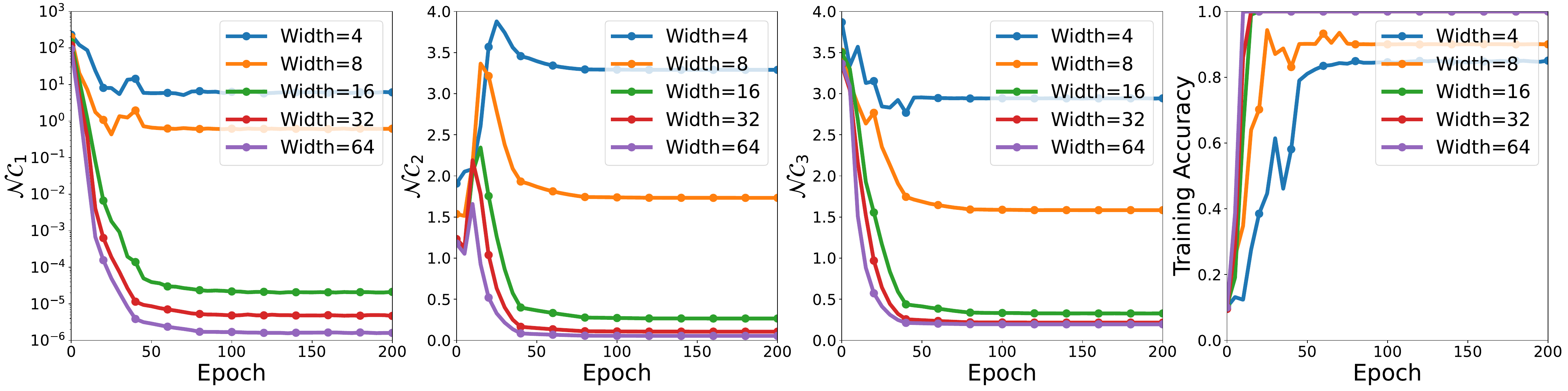}
    \caption{\textbf{Validation of UFM on ResNet with varying network width.} $\mathcal{NC}$ metrics and training accuracy of ResNet-18 networks of various widths on CIFAR10 with $n=200$ over 200 epochs.}
    \label{fig:nc_width}
\end{figure}

\subsection{Feature normalization for improved training speed and representation quality}\label{sec:fast_train}
We now investigate the benefits of using feature normalization for improving training speed and representation quality. First, we optimize Problem \eqref{eq:ce-loss func} and compare it to the regularized counterpart in \cite{zhu2021geometric} in the UFM. The results are shown in \Cref{fig:faster_training_ufm}.hao We can see that normalizing features over the sphere consistently results in reaching perfect classification and greater feature collapse (i.e., smaller $\mc {NC}_1$) quicker than penalizing the features. To demonstrate that these behaviors are reflected in training practical deep networks, we train both ResNet-18 and ResNet-50 architectures on a reduced CIFAR100 \cite{krizhevsky2009learning} dataset with $N=3000$ total samples, comparing the training accuracy and metrics of \NC\ with and without feature normalization. The results are shown in \Cref{fig:faster_training}. \\
\begin{figure}[t]
    \centering
    \includegraphics[width=1\textwidth]{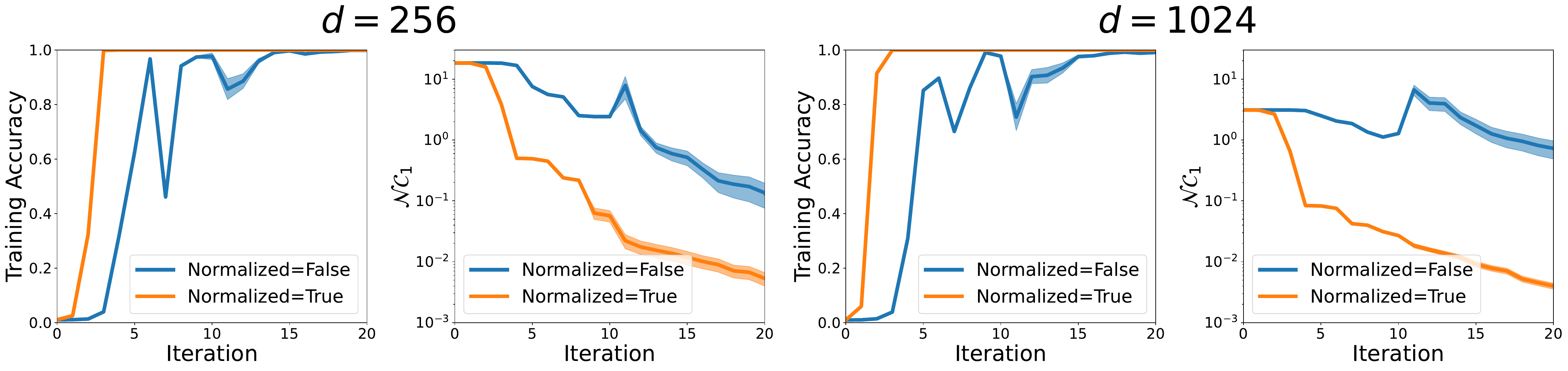}
    \caption{\textbf{Faster training/feature collapse of UFM with feature normalization.} Average (deviation denoted by shaded area) training accuracy and $\mathcal{NC}_1$ of UFM over 10 trials of (Riemmanian) gradient descent with backtracking line search. We set $K=100$ classes, $n=30$ samples per class.}
    \label{fig:faster_training_ufm}
\end{figure}

\begin{figure}[t]
    \centering
    \includegraphics[width=1\textwidth]{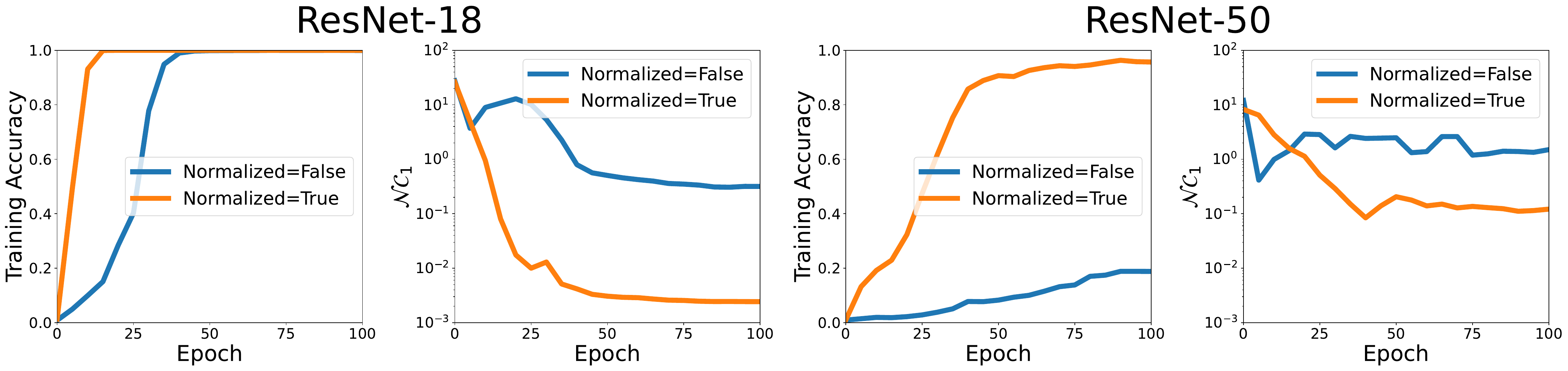}
    \caption{\textbf{Faster training/feature collapse with ResNet on CIFAR100 with feature normalization.} Training accuracy and $\mathcal{NC}_1$ of ResNet-18 and ResNet-50 on CIFAR100 with $n=30$ over 100 epochs.}
    \label{fig:faster_training}
\end{figure}
\noindent From \Cref{fig:faster_training} (left), we can see that for the ResNet-18 network, we reach perfect classification of the training data about 10-20 epochs sooner by using feature normalization compared to that of the unconstrained formulation. From \Cref{fig:faster_training} (right), training the ResNet-50 network without feature normalization for 100 epochs shows slow convergence with poor training accuracy, whereas using feature normalization arrives at above 90\% training accuracy in the same number of epochs. By keeping the size of the dataset the same and increasing the number of parameters, it is reasonable that optimizing the ResNet-50 network is more challenging due to the higher degree of overparameterization, yet this effect is mitigated by using feature normalization. \\~\\
At the same time, for both architectures, using feature normalization leads to greater feature collapse (i.e., smaller $\mc {NC}_1$) compared to that of the unconstrained counterpart. As shown in recent work \cite{papyan2020prevalence,zhou2022optimization,galanti2022on} , better \NC\ often leads to better generalization performance, as corroborated by \Cref{sec:generalization}. Last but not least, we believe the benefits of feature normalization are not limited to the evidence that we showed here, as it could also lead to better robustness \cite{chan2020deep,yu2020learning} that is worth further exploration.

\subsection{Neural collapse occurs independently of training algorithms under UFM}\label{sec:cg_trm}

In \Cref{sec:fast_train}, we demonstrated that optimizing Problem \eqref{eq:ce-loss func}, which corresponds to the feature normalized UFM, results in quicker training and feature collapse as opposed to the regularized UFM formulation, as shown in \Cref{fig:faster_training_ufm}. To show that this phenomenon is independent of the algorithm used, we additionally test the conjugate gradient (CG) method \cite{absil2009optimization} as well as the trust-region method (TRM) \cite{absil2009optimization} to solve \eqref{eq:ce-loss func} with the same set-up as in \Cref{fig:faster_training_ufm}. While the Riemannian conjugate gradient method is also a first order method like gradient descent, the Riemannian trust-region method is a second order method, so the convergence speed is much faster compared to Riemannian gradient descent or conjugate gradient method. The results are shown in \Cref{fig:faster_training_ufm_cg} and \Cref{fig:faster_training_ufm_trm} for the CG method and TRM respectively. \\~\\
We see that optimizing the feature normalized UFM with CG gives similar results to using GD, whereas optimizing the feature normalized UFM using TRM results in an even greater gap in convergence speed to the global solutions, when compared with optimizing the regularized counterpart using TRM. These results suggest that the benefits of feature normalization are not limited to vanilla gradient descent or even first order methods.
\begin{figure}[t]
    \centering
    \includegraphics[width=\textwidth]{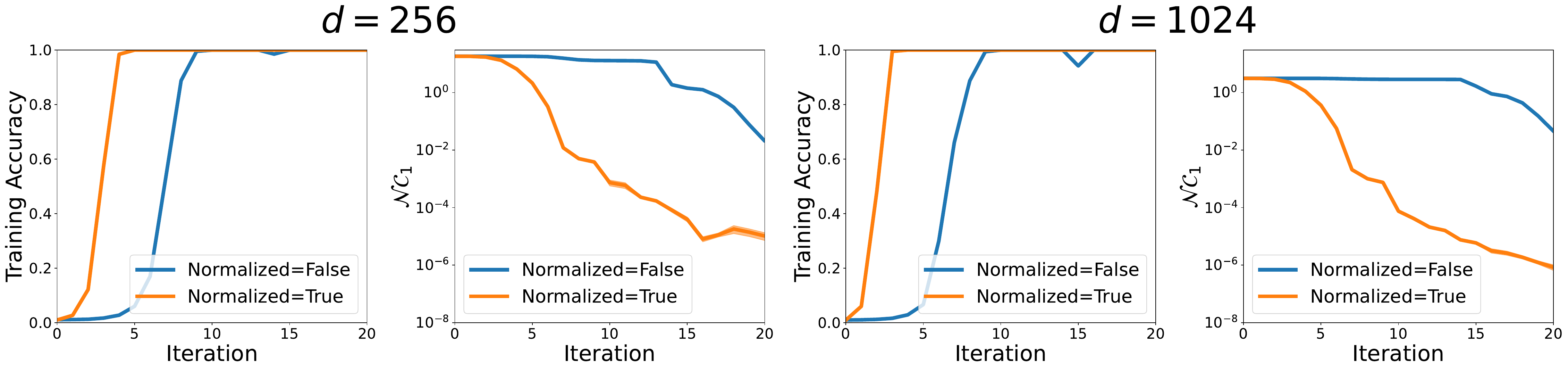}
    \caption{\textbf{Faster training/feature collapse of UFM with feature normalization with CG.} Average (deviation denoted by shaded area) training accuracy and $\mathcal{NC}_1$ of UFM over 10 trials of (Riemmanian) conjugate gradient method. We set $K=100$ classes, $n=30$ samples per class.}
    \label{fig:faster_training_ufm_cg}
\end{figure}

\begin{figure}[t]
    \centering
    \includegraphics[width=\textwidth]{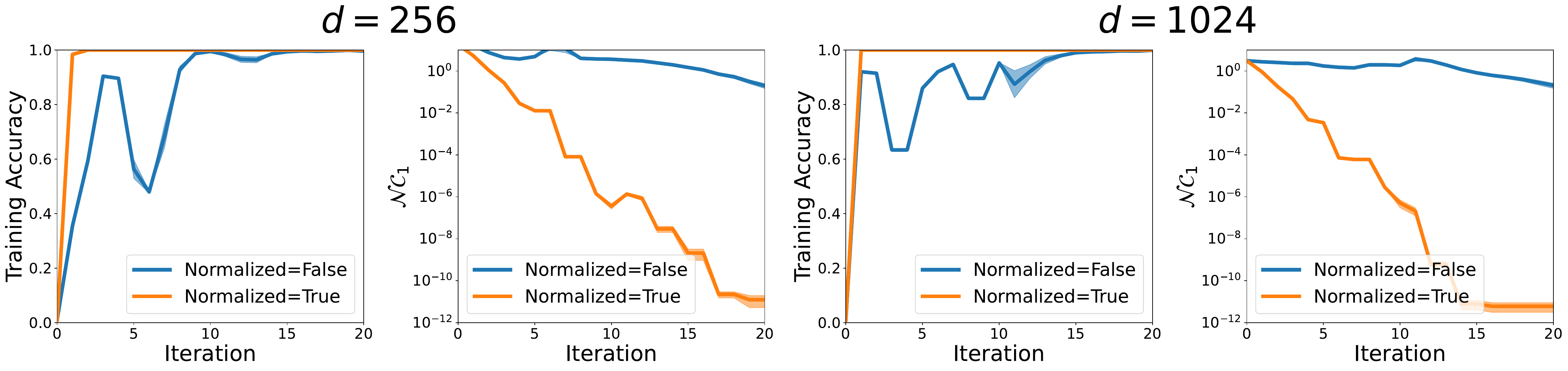}
    \caption{\textbf{Faster training/feature collapse of UFM with feature normalization with TRM.} Average (deviation denoted by shaded area) training accuracy and $\mathcal{NC}_1$ of UFM over 10 trials of (Riemmanian) trust-region method. We set $K=100$ classes, $n=30$ samples per class.}
    \label{fig:faster_training_ufm_trm}
\end{figure}

\subsection{Feature normalization generalizes better than regularization}\label{sec:generalization}

In \Cref{sec:fast_train}, we showed that using feature normalization over regularization improves training speed and feature collapse when training increasingly overparameterized ResNet models on a small subset of CIFAR100. We now demonstrate that feature normalization leads to better generalization than regularization. \\~\\
We train a ResNet-18 and ResNet-50 model on the \emph{entirety} of the CIFAR100 training split without any data augmentation for 100 epochs, and test the accuracy and $\mathcal{NC}_1$ metric on the standard test split. The results are shown in \Cref{tab:better_generalization}. We immediately see that using feature normalization gives both better test accuracy and test feature collapse than using regularization. Furthermore, the test generalization performance is coupled with the degree of feature collapse, supporting the claim that better \NC\ often leads to better generalization performance. Finally, as we have trained both ResNet architectures with the same set-up and number of epochs, there is a substantial drop in performance (both test accuracy and test $\mathcal{NC}_1$) of the regularized ResNet-50 model compared to the ResNet-18. However, using feature normalization, this effect is mostly mitigated, suggesting that feature normalization is more robust compared to regularization and effective for generalizing highly overparameterized models on fixed-size datasets.

\begin{table}[t]
    \centering
    \vspace{1em}
    \renewcommand{\arraystretch}{1.1}
    \scalebox{1.1}{
    \begin{tabular}{|l|c|c|c|c|} \hline
        & \multicolumn{2}{c|}{ResNet-18} & \multicolumn{2}{c|}{ResNet-50} \\ \hline
         & Test Accuracy & Test $\mathcal{NC}_1$ & Test Accuracy & Test $\mathcal{NC}_1$ \\ \hline
        Regularization & 55.3\% & 3.838 & 48.9\% & 4.486 \\ \hline
        Normalization & \textbf{58.6\%} & \textbf{3.143} & \textbf{56.4\%} & \textbf{3.127} \\ \hline
    \end{tabular}}
    \caption{\textbf{Better generalization and test feature collapse with ResNet on CIFAR100 with feature normalization.} Test accuracy and test $\mathcal{NC}_1$ of ResNet-18 and ResNet-50 on CIFAR100. }
    \label{tab:better_generalization}
    \renewcommand{\arraystretch}{1}
\end{table}

\subsection{Investigating the effect of the temperature parameter $\tau$}\label{sec:tau}
In \Cref{sec:formulation}, although the temperature parameter $\tau > 0$ does not affect the global optimality and critical points, it does affect the training speed of specific learning algorithms and hence test performance. In all experiments in \Cref{sec:experiments}, we set $\tau=1$ and have not discussed in detail the effects of $\tau$ in practice.\\~\\
However, as mentioned in \cite{graf2021dissecting}, $\tau$ has important side-effects on optimization dynamics and must be carefully tuned in practice. Hence, we now present a brief study of the temperature parameter $\tau$ when optimizing the problem \eqref{eq:ce-loss func} under the UFM and training a deep network in practice.
To begin, we consider the UFM formulation. We first note that $\tau$ does not affect the theoretical global solution or benign landscape of the UFM (although it does affect the attained theoretical lower bound, see \Cref{thm:ce optim} and \Cref{fig:lower_bound}). However, it does impact the rate of convergence to neural collapse as well as the attained numerical values of the $\mathcal{NC}$ metrics. To see this, we apply (Riemannian) gradient descent with backtracking line search to Problem \eqref{eq:ce-loss func} for various settings of $\tau$. These results are shown in \Cref{fig:tau_ufm}. 
\begin{figure}[t]
    \centering
    \includegraphics[width=\textwidth]{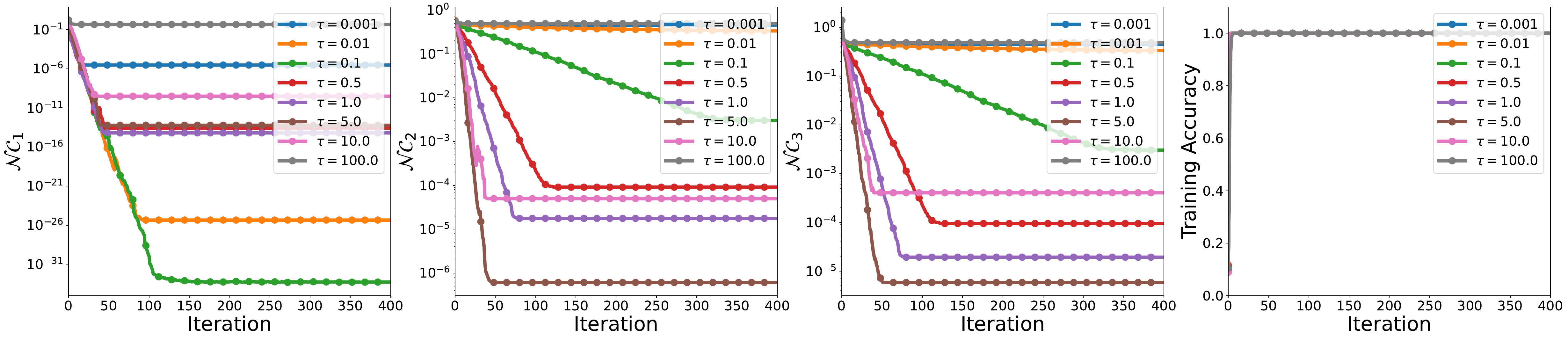}
    \includegraphics[width=\textwidth]{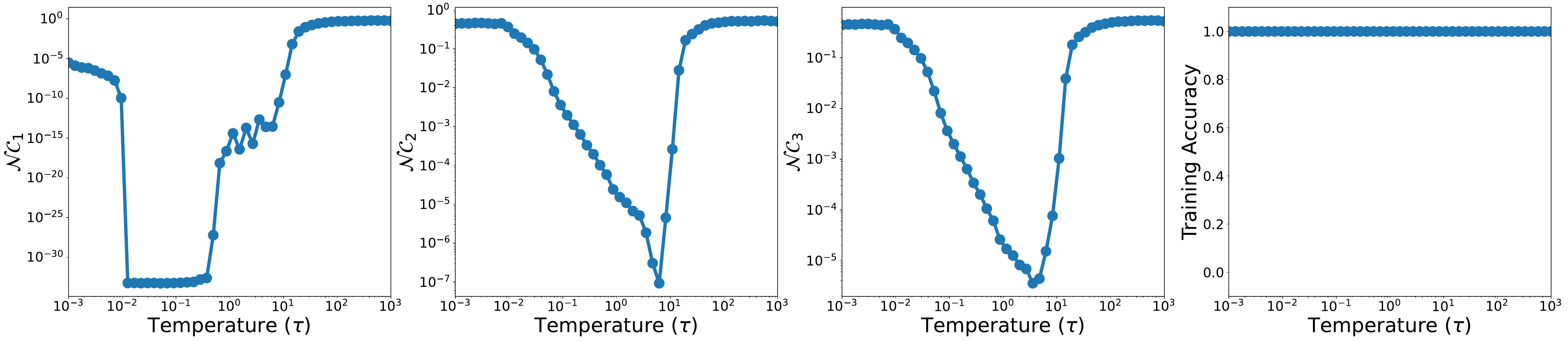}
    \caption{\textbf{Effect of temperature parameter on collapse and training accuracy of UFM.} $K = 10$ classes, $n = 5$ samples per class, $d = 32$. Top row: Average $\mathcal{NC}$ metrics and training accuracy of UFM over 20 trials for various settings of $\tau$ with respect to each iteration of (Riemannian) gradient descent. Bottom row: Final average $\mathcal{NC}$ metrics and training accuracy of UFM over 20 trials for various settings of $\tau$.}
    \label{fig:tau_ufm}
\end{figure}
First, it is evident that for all tested $\tau$ values, we achieve perfect classification in a similar number of iterations. Furthermore, the rate of convergence of $\mathcal{NC}_1$ is somewhat the same for most settings of $\tau$, and we essentially have feature collapse for most settings of $\tau$. On the other hand, it appears that the rate of convergence of $\mathcal{NC}_2$ and $\mathcal{NC}_3$ are dramatically affected by $\tau$, with values in the range of 1 to 10 yielding the greatest collapse. This aligns with the choice of the temperature parameter in the experimental section of \cite{graf2021dissecting}, where the equivalent parameter is set $\rho = 1/\sqrt{0.1} \approx 3.16$. \\~\\
We now look to the setting of training practical deep networks. We train a feature normalized ResNet-18 architecture on CIFAR-10 for various settings of $\tau$. The results are shown in \Cref{fig:tau}.
\begin{figure}[t]
    \centering
    \includegraphics[width=\textwidth]{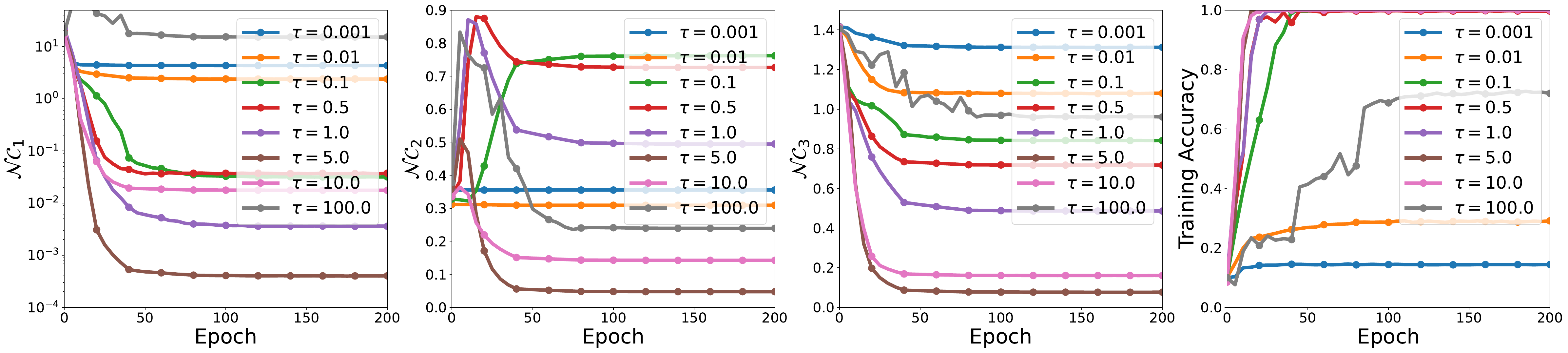}
    \includegraphics[width=\textwidth]{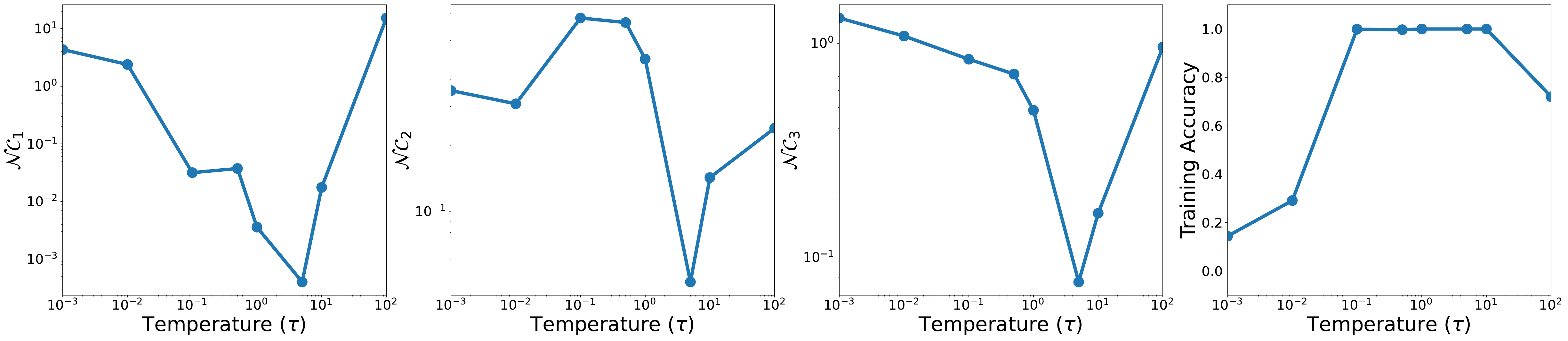}
    \caption{\textbf{Effect of temperature parameter on collapse and training accuracy of ResNet.} Top row: Average $\mathcal{NC}$ metrics and training accuracy of ResNet-18 on CIFAR-10 with $n=100$ for various settings of $\tau$ with respect to each epoch over 200 epochs. Bottom row: Final average $\mathcal{NC}$ metrics and training accuracy of ResNet-18 on CIFAR-10 with $n=100$ for various settings of $\tau$.}
    \label{fig:tau}
\end{figure}
One immediate difference from the UFM formulation is that we arrive at perfect classification of the training data for a particular range of values for $\tau$ (from about $0.1$ to $10$) but not for all settings of $\tau$. Within this range, we can see that values of $\tau$ around 1 to 10 lead to the fastest training, and values close to $\tau=5$ lead to the greatest collapse in all $\mathcal{NC}$ metrics, as was the case with the UFM. All this evidence suggests that $\tau=1$ is not the optimal setting of the temperature parameter for either the UFM or ResNet, particularly when measuring $\mathcal{NC}_2$ and $\mathcal{NC}_3$, and instead $\tau=5$ may perform better. In the practical experiments of the main text, however, we mainly focused on training speed and feature collapse, and for these purposes it appears that the $\tau$ parameter can be set in a fairly nonstringent manner.

\subsection{Exploration of benign global landscapes of other commonly-used losses}\label{sec:other-losses}
Although the cross-entropy (CE) loss studied in this work is arguably the most common loss function for deep classification tasks, it is not the only one. Some other commonly used loss functions include focal loss (FL) \cite{lin2017focal}, label smoothing (LS) \cite{szegedy2015rethinking}, and supervised contrastive (SC) loss \cite{khosla2020supervised}, each of which has demonstrated various benefits over vanilla CE. In this section, we briefly explore the empirical global landscape of these losses under the UFM with normalized features (and classifiers). Specifically, we consider the problem
\begin{align} 
     \min_{\mb W, \mb H} &\; f(\mb W,\mb H) \;:=\; \frac{1}{N} \sum_{k=1}^K \sum_{i=1}^{n} \mc L \paren{ \tau \mb W^\top \mb h_{k,i}, \mb y_k }, \label{eq:general loss func} \\
     \;\text{s.t.} & \quad \mb H \in \mc {OB}(d,N),\;  \mb W \in\; \mc {OB}(d,K). \nonumber 
\end{align}
where $\mathcal{L}$ is either the focal loss or label smoothing loss. \\~\\
First, we consider the focal loss, defined as
\begin{align*}
    \mathcal{L}_{\mathrm{FL}}(\mb z, \mb y_k) = -\left(1 - \frac{\exp(z_k)}{\sum_{\ell=1}^K \exp(z_{\ell})}\right)^{\gamma} \log\left(\frac{\exp(z_k)}{\sum_{\ell=1}^K \exp(z_{\ell})}\right)
\end{align*}
where $\gamma \geq 0$ is the \emph{focusing} parameter (with $\gamma = 0$, we recover ordinary CE). As seen in \Cref{fig:fl_lower_bound}, using gradient descent with random initialization on the focal loss, we achieve neural collapse over a range of settings of $K$ and $\tau$. Characterizing the global solutions of \eqref{eq:general loss func} for the focal loss and a general landscape analysis are left as future work.\\~\\
\begin{figure}[t]
    \centering
    \includegraphics[width=1\textwidth]{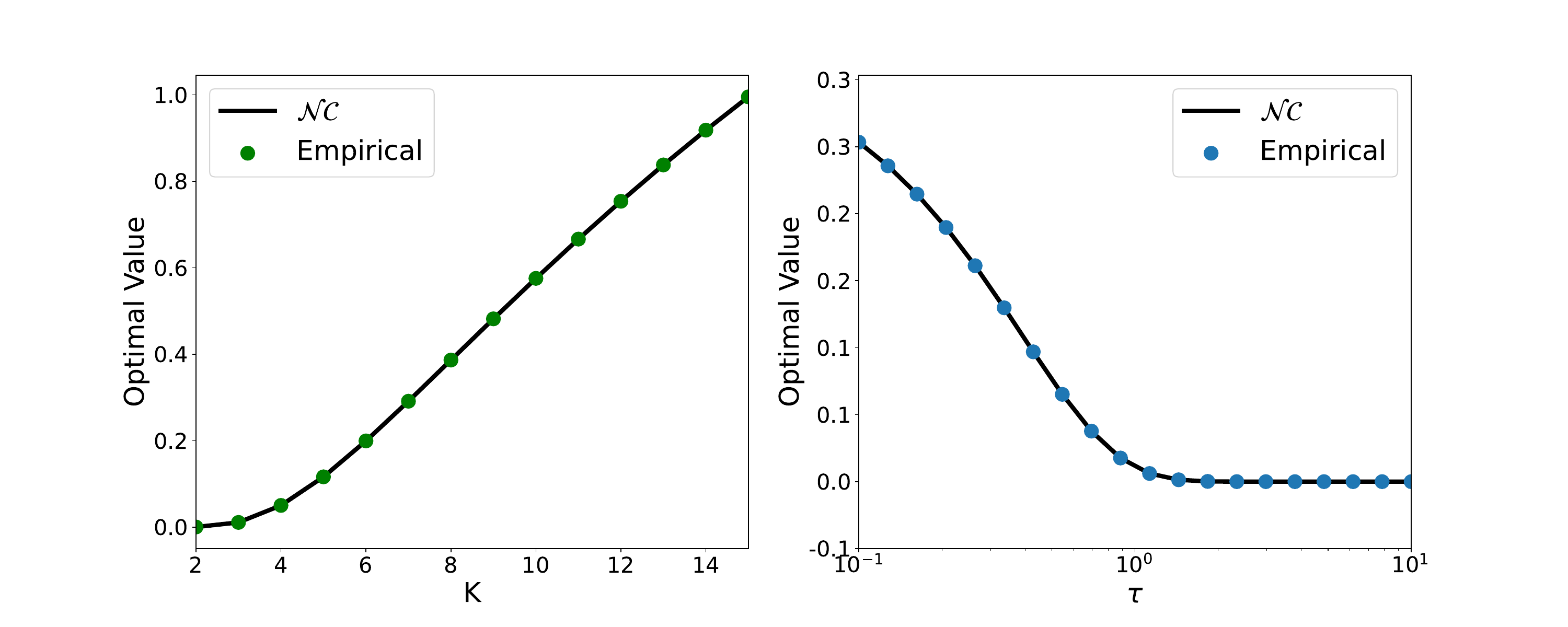}
    \caption{\textbf{Global optimization of focal loss (FL) with $\gamma = 3$ under UFM with $d=16$ and $n=3$.} Black line refers to theoretical value of \eqref{eq:general loss func} at $\mathcal{NC}$ solutions. Empirical values found using gradient descent with random initialization. Left: Lower bound against number of classes $K$ while fixing $\tau=1$. Right: Lower bound against temperature $\tau$ while fixing $K=3$. The same empirical values are achieved over many trials.}
    \label{fig:fl_lower_bound}
\end{figure}
Next, we consider the label smoothing loss, defined as
\begin{align*}
    \mathcal{L}_{\mathrm{LS}}(\mb z, \mb y_k) = - \left(1 - \frac{K-1}{K}\alpha\right) \log\left(\frac{\exp(z_k)}{\sum_{\ell=1}^K \exp(z_{\ell})}\right) - \frac{\alpha}{K}\sum_{j\neq k} \log\left(\frac{\exp(z_j)}{\sum_{\ell=1}^K \exp(z_{\ell})}\right)
\end{align*}
where $\alpha \geq 0$ is the \emph{smoothing} parameter (with $\alpha = 0$, we recover ordinary CE). As seen in \Cref{fig:ls_lower_bound}, for small enough $\tau$ we achieve neural collapse, but for larger $\tau$, we do not. In fact, the global solutions of the label smoothing loss are not neural collapse for large enough $\tau$. To see this, let $(\mb W_1, \mb H_1)$ denote a $\mathcal{NC}$ solution, and let $(\mb W_2, \mb H_2)$ denote a solution where $\mb W_2 = \mb a \mb 1_K^\top$ and $\mb H_2 = \mb a \mb 1_N^\top$, where $\mb a$ is unit-norm. It is easy to compute that
\begin{align*}
    f(\mb W_1, \mb H_1) = \log\left(1 + (K-1)\exp\left(-\frac{K\tau }{K-1}\right)\right) + \alpha \tau
\end{align*}
so $f(\mb W_1, \mb H_1) \rightarrow \infty$ as $\tau \rightarrow \infty$, whereas $f(\mb W_2, \mb H_2) = \log(K)$ is independent of $\tau$. Again, characterizing the global solutions of \eqref{eq:general loss func} for label smoothing and a general landscape analysis are left as future work.
\begin{figure}[t]
    \centering
    \includegraphics[width=1\textwidth]{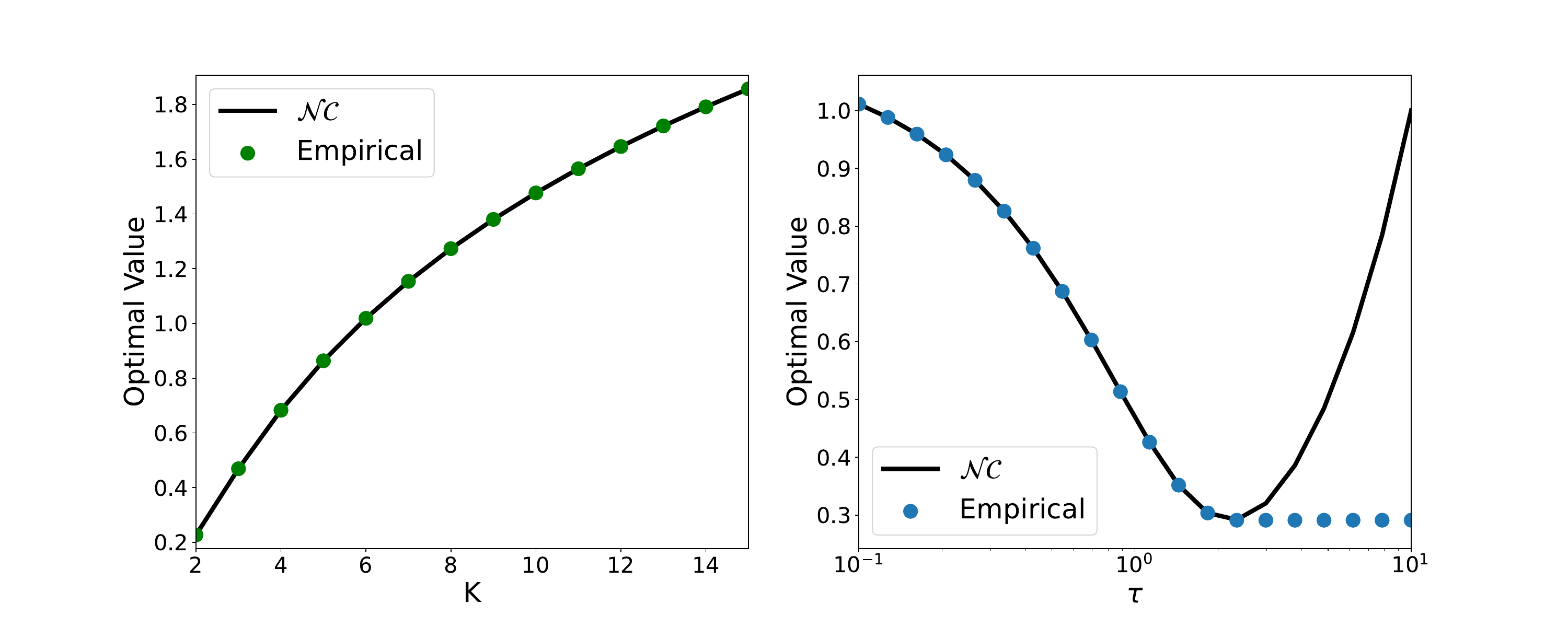}
    \caption{\textbf{Global optimization of label smoothing (LS) with $\alpha = 0.1$ under UFM with $d=16$ and $n=3$.} Black line refers to value of \eqref{eq:general loss func} at $\mathcal{NC}$ solutions. Empirical values found using gradient descent with random initialization. Left: Lower bound against number of classes $K$ while fixing $\tau=1$. Right: Lower bound against temperature $\tau$ while fixing $K=3$. The same empirical values are achieved over many trials.}
    \label{fig:ls_lower_bound}
\end{figure}
\\~\\
Finally, we look to the supervised contrastive loss. Unlike the other losses, we do not have classifier $\mb W$ when training, so we instead have the problem
\begin{align}
    \min_{\mb H} f(\mb H) := &-\frac{1}{N(n-1)} \sum_{i=1}^N \sum_{\substack{j \neq i \\ y_j = y_i}} \log \left( \frac{\exp(\tau^2 \mb h_i^\top \mb h_j)}{\sum_{\ell\neq i} \exp(\tau^2 \mb h_i^\top \mb h_{\ell})}\right) \label{eq:sc loss func} \\
     \;\text{s.t.} & \quad \mb H \in \mc {OB}(d,N) \nonumber.
\end{align}
We note that the loss as written above computes the loss over the entire dataset, as opposed to computing over all minibatches of a fixed size as in \cite{graf2021dissecting}. As seen in \Cref{fig:sc_lower_bound}, using gradient descent with random initialization on the supervised contrastive loss, we achieve neural collapse over a range of settings of $K$ and $\tau$. In fact, it is proven in \cite{graf2021dissecting} that the global minimizers of \eqref{eq:sc loss func} are $\mathcal{NC}$ solutions. However, an understanding of the global landscape requires further exploration and is left as future work.
\begin{figure}[t]
    \centering
    \includegraphics[width=1\textwidth]{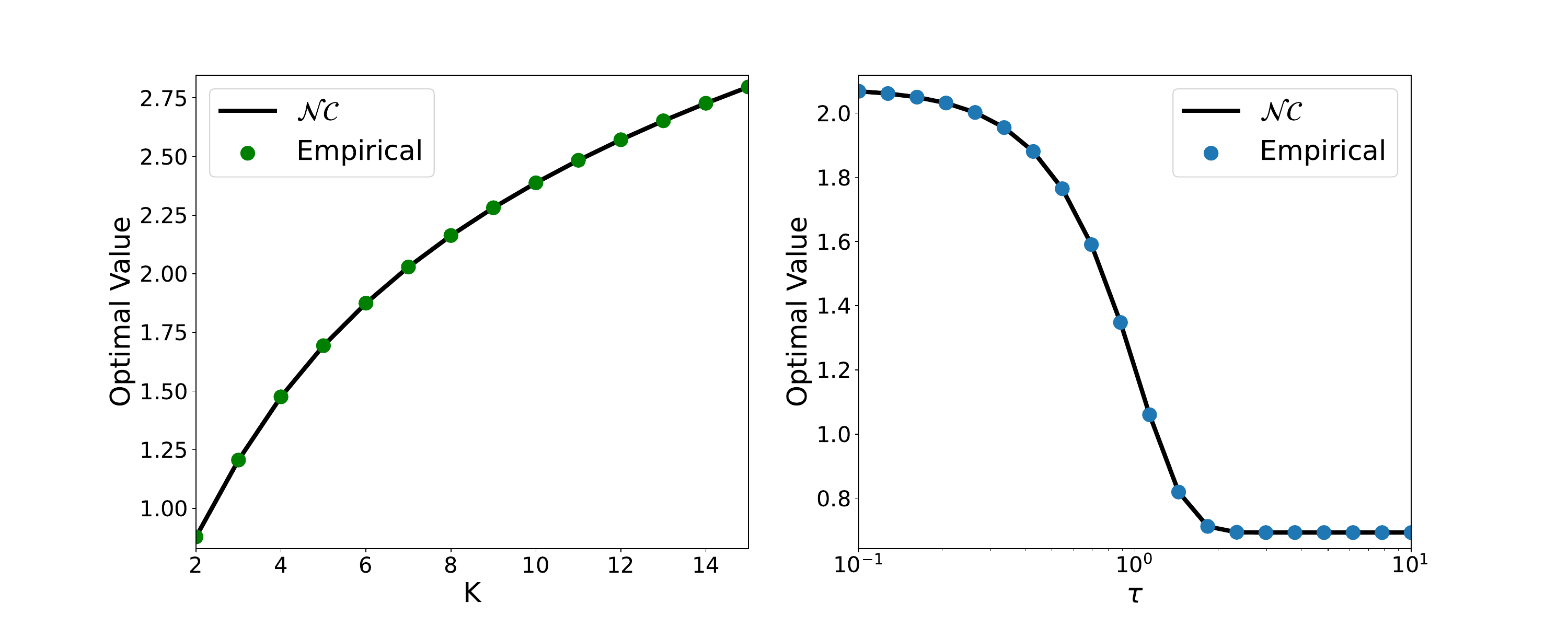}
    \caption{\textbf{Global optimization of supervised contrastive (SC) loss under UFM with $d=16$ and $n=3$.} Black line refers to value of \eqref{eq:sc loss func} at $\mathcal{NC}$ solutions. Empirical values found using gradient descent with random initialization. Left: Lower bound against number of classes $K$ while fixing $\tau=1$. Right: Lower bound against temperature $\tau$ while fixing $K=3$. The same empirical values are achieved over many trials.}
    \label{fig:sc_lower_bound}
\end{figure}

\section{Conclusion \& Discussion}\label{sec:conclusion}
In this work, motivated by the common practice of feature normalization in modern deep learning, we study the prevalence of the \NC\ phenomenon when last-layer features and classifiers are constrained on the sphere. Based upon the assumption of the UFM, we formulate the problem as a Riemannian optimization problem over the product of sphere (i.e., oblique manifold). We showed that the loss function is a strict saddle function over the manifold with respect to the last-layer features and classifiers, with no other spurious local minimizers. We demonstrated that this phenomenon occurs for overparameterized deep network training, and show the benefits of feature normalization in terms of training speed, learned representation quality, and generalization, both for the UFM and for practical deep networks on classification tasks. We conclude by placing our work in the context of existing literature, and we briefly discuss several exciting future research directions, motivated both by previous work as well as the work presented here.

\paragraph{Global optimality of \NC\ under UFM.} The seminal works \cite{papyan2020prevalence,han2021neural} inspired many recent theoretical studies of the \NC\ phenomenon. Because the training loss of a deep neural network is highly nonlinear, most works simplify the analysis by assuming unconstrained feature models (UFM) \cite{mixon2020neural,zhu2021geometric} or layer peeled models \cite{fang2021layer}. It basically assumes that the network has infinite expression power so that the features can be reviewed as free optimization variables. Based upon the UFM, \cite{lu2020neural} is the first work justifying the global optimality of \NC\ and uniformity based upon a CE loss with normalized features, although their study is quite simplified in the sense that they assume each class has only one training sample. The work \cite{fang2021layer} provided global optimality analysis for the CE loss with constrained features under more generic settings, and they also studied the case when the training samples are imbalanced in each class. However, this work constrains the sum of feature norms, whereas in our work, we constrain the features independently. The follow-up work \cite{ji2021unconstrained} extended the analysis to the unconstrained setting without any penalty. Additionally, motivated by the commonly used weight decay on network parameters, the work \cite{zhu2021geometric} justifies the global optimality of \NC\ for the CE loss under the unconstrained formulation, with penalization on both the features and classifiers. Its companion work \cite{zhou2022optimization} extended the analysis to the MSE loss. Under the same assumption, other work \cite{graf2021dissecting} studied both the CE and SC losses with features lying in the ball, proving that the only global solutions satisfy \NC\ properties. Our setting is most closely related to that in \cite{graf2021dissecting}, except we explicitly constrain the features on the sphere. Moreover, the work \cite{tirer2022extended} studied the setting beyond the simple UFM, showing that, even for a three-layer nonlinear network, the \NC\ solutions are the only global solutions with the MSE training loss. Motivated by feature normalization, as future work one could study a three-layer nonlinear network where the penultimate output is projected onto the sphere.

\paragraph{Benign global landscape and learning dynamics under UFM.} Since the training loss is highly nonconvex even under the UFM, merely studying global optimality is not sufficient for guaranteeing efficient global optimization. More recent works address this issue by investigating the global landscape properties and learning dynamics of specific training algorithms. More specifically, under the UFM, \cite{zhu2021geometric,zhou2022optimization} showed that the optimization landscapes of CE and MSE losses have benign global optimization landscapes, in the sense that every local minimizer satisfies \NC\ properties and the remaining critical points are strict saddles with negative curvatures. These works considered the unconstrained formulation with regularization on both features and classifiers. In comparison, our work studies the benign global landscape with features and classifiers constrained over the product of spheres - to our knowledge, this is the first work to study the global landscape of the constrained formulation. We have also empirically demonstrated benign landscapes for other losses such as focal loss and supervised constrastive loss, but further development is required for a theoretical analysis of the benign landscape. On the other hand, there is another line of works studying the implicit bias of learning dynamics under UFM \cite{mixon2020neural,rangamani2021dynamics,rangamani2022deep,poggio2020implicit,poggio2020explicit,ji2021unconstrained}, showing that the convergent direction is along the direction of the minimum-norm separation problem for both CE and MSE losses.

\paragraph{The empirical phenomena of \NC\ and feature engineering}
Although the seminal works \cite{papyan2020prevalence} and \cite{han2021neural} are the first to summarize the empirical prevalence of \NC\ for the commonly used CE and MSE losses respectively, the idea of designing features with intra-class compactness and inter-class separability has a richer history. More specifically, in the past many loss functions, such as center loss \cite{wen2016discriminative}, large-margin softmax (L-Softmax) loss \cite{liu2016large}, and its variants \cite{ranjan2017l2,liu2017sphereface,wang2018cosface,deng2019arcface} were designed with similar goals for the task of visual face recognition. Additionally, related works \cite{pernici2019maximally,pernici2021regular} introduced similar ideas of learning maximal separable features by fixing the linear classifiers with a simplex-shaped structure. Furthermore, the work \cite{zhou2022optimization} demonstrated that better collapse could potentially lead to better generalization, which is corroborated by our experiments with normalized features. However, learning neural collapsed features could also easily lead to overfitting \cite{zhang2016understanding} and vulnerability to data corruptions \cite{zhu2021geometric}. Additionally, the collapse of the feature dimension could cause the loss of intrinsic structure of input data, making the learned features less transferable \cite{kornblith2021why}. In self-supervised learning, recent works promote feature diversity and uniformity via contrastive learning \cite{chen2020simple,wang2020understanding}. In contrast, a line of recent work proposed to learn diverse while discriminative representation by designing a loss that maximizes the coding rate reduction \cite{yu2020learning,chan2021redunet}. Instead of collapsing the features to a single dimension, these works promote within-class diversity while maintaining maximum between-class separability. As such, it leads to better robustness and transferability. Since the coding rate is monotonic in the scale of the features, these works also normalize features on the sphere, making it a suitable problem to be studied over the manifold similar to our work. A geometric analysis of the coding rate reduction landscape is a topic of future research. 

\paragraph{Benefits of feature normalization}
The works \cite{ranjan2017l2,liu2017sphereface,wang2018cosface,deng2019arcface} first introduced feature normalization and demonstrated its advantages for learning more separable/discriminative features, which well motivates our study in this work. In fact, we demonstrated in \Cref{fig:low_dim_features} and \Cref{tab:low_dim_acc} that features learned on the sphere are more separable in low dimensions. However, as seen in our experimental results, the benefits of feature normalization extend beyond better separability. We have empirically demonstrated that using feature normalization leads to faster training and better generalization in some settings. Further investigation is necessary to gain a better understanding of the role of feature normalization in these phenomena.

\paragraph{Large number of classes $K \gg d$}
Lastly, in this work we have presented both theoretical and experimental results either in the case $d \geq K$ or $d > K$. In many applications, this is quite reasonable, since the number of classes isn't too large and we can pick $d$ accordingly in the network architecture. However, in the case of self-supervised constrastive learning, we do not have label information to assign training samples to common class labels, and therefore the number of classes $K$ can grow very large. In other applications such as recommendation systems \cite{covington2016deep} and document retrieval \cite{chang2019pre} the number of classes can also grow very large. In \Cref{fig:low_dim_features} and \Cref{tab:low_dim_acc}, we can see that feature normalization under the UFM gives better separability and classification accuracy when $K \gg d$, making feature normalization a good candidate for improving performance on these problems. However, our theory does not readily extend to this case. Although feature collapse is still a reasonable notion, we can not form a simplex ETF when $d < K+1$. Instead, we can measure the \emph{uniformity} of features over the sphere in place, and in fact the work \cite{wang2020understanding} shows that feature alignment (collapse) and uniformity are asymptotically minimized for contrastive loss. A finite sample analysis of the $K \gg d$ using the uniformity metric would be a meaningful extension to our work and of great interest for the aforementioned applications.

\section*{Acknowledgement}

Can Yaras and Qing Qu acknowledge support from U-M START \& PODS grants, NSF CAREER CCF 2143904, NSF CCF 2212066, NSF CCF 2212326, and ONR N00014-22-1-2529 grants. Peng Wang and Laura Balzano acknowledge support from ARO YIP W911NF1910027, AFOSR YIP FA9550-19-1-0026, and NSF CAREER CCF-1845076. Zhihui Zhu acknowledges support from NSF grants CCF-2240708 and CCF-2241298. We would like to thank Zhexin Wu (ETH), Yan Wen (Tsinghua), and Pengru Huang (UMich) for fruitful discussion through various stages of the work.


{\small
\bibliographystyle{abbrv}
\bibliography{biblio/optimization,biblio/learning}
}

\newpage 
\appendix

\newpage
\onecolumn
\par\noindent\rule{\textwidth}{1pt}
\begin{center}
{\Large \bf Appendix}
\end{center}
\vspace{-0.1in}
\par\noindent\rule{\textwidth}{1pt}

\paragraph{Organization of the appendices.}  The appendix is organized as follows. In \Cref{subsec:tools}, we provide some preliminary tools for analyzing our manifold optimization problem. Based upon this, the proof of \Cref{thm:ce optim} and the proof of \Cref{thm:ce landscape} are provided in \Cref{app:thm global} and \Cref{app:thm landscape}, respectively. 

\paragraph{Notations.} Before we proceed, let us first introduce the notations that will be used throughout the appendix. Let $\R^n$ denote $n$-dimensional Euclidean space and $\|\cdot\|_2$ be the Euclidean norm. We write matrices in bold capital letters such as $\mb A$, vectors in bold lower-case such as $\mb a$, and scalars in plain letters such as $a$. Given a matrix $\mb A \in \R^{d\times K}$, we denote by $\mb a_k$ its $k$-th column, $\mb a^i$ its $i$-th row, $a_{ij}$ its $(i,j)$-th element, and $\|\mb A\|$ its spectral norm. We use $\diag(\mb A)$ to denote a vector that consists of diagonal elements of $\mb A$ and $\ddiag(\mb A)$ to denote a diagonal matrix whose diagonal elements are the diagonal ones of $\mb A$. We use $\diag(\mb a)$ to denote a diagonal matrix whose diagonal is $\mb a$. Given a positive integer $n$, we denote by $[n]$ the set $\{1,\dots,n\}$. We denote the unit sphere in $\R^d$ by $\bb S^{d-1} := \{\mb x \in \R^d: \|\mb x\|_2 = 1\}$.




\section{Preliminaries}\label{subsec:tools}

In this section, we first review some basic aspects of the Riemannian optimization and then compute the derivative of the CE loss.

\subsection{Riemannian Derivatives} \label{subsec:Rie tools}

According to \cite[Chapter 3 \& 5]{boumal2022intromanifolds} and \cite{hu2020brief,absil2009optimization}, the tangent space of a general manifold $\mc M \subseteq \bb R^d$ at $\mb x $, denoted by $\mr{T}_{\mb x}\mc M$, is defined as the set of all vectors tangent to $\mc M$ at $\mb x$. Based on this, the Riemannian gradient $\grad f$ of a function $f$ at $\mb x$ is a unique vector in $\mr{T}_{\mb x}\mc M$ satisfying
\begin{align*}
    \innerprod{\grad f}{\mb \xi} \;=\; Df(\mb x) [\mb \xi],\quad \forall\ \mb \xi \in \mr{T}_{\mb x}\mc M.
\end{align*}
where $Df(\mb x) [\mb \xi]$ is the derivative of $f(\gamma(t))$ at $t=0$, $\gamma(t)$ is any curve on the manifold that satisfies $\gamma(0) = \mb x$ and $\dot{\gamma}(0) = \mb \xi$. The Riemannian Hessian $\Hess f(\mb x)$ is a mapping from the tangent space $\mr{T}_{\mb x}\mc M$ to the tangent space $\mr{T}_{\mb x}\mc M$ with
\begin{align*}
    \Hess f(\mb x)[\mb \xi] \;=\; \wt{\nabla}_{\mb \xi} \grad f(\mb x),
\end{align*}
where $\wt{\nabla}$ is the Riemannian connection. For a function $f$ defined on the manifold $\mc M$, if it can be extended smoothly to the ambient Euclidean space, we have 
\begin{align*}
    \grad f(\mb x) \;&=\; \mc P_{ \mr{T}_{\mb x}\mc M } \paren{  \nabla f(\mb x) } , \\
    \Hess f(\mb x)[\mb \xi] \;&=\; \mc P_{ \mr{T}_{\mb x}\mc M } \paren{ D\grad f(\mb x)[\mb \xi] }.
\end{align*}
where $D$ is the Euclidean differential, and $\mc P_{ \mr{T}_{\mb x}\mc M }$ is the projection on the tangent space $\mr{T}_{\mb x}\mc M$. According to \cite[Example 3.18]{boumal2022intromanifolds}, if $\mc M = \bb S^{p-1}$, then the tangent space and projection are
\begin{align*}
    \mr{T}_{\mb x} \bb S^{p-1} \;=\; \Brac{ \mb z \in \bb R^p  \mid \mb x^\top \mb z = 0 }, \quad \mc P_{\mr{T}_{\mb x} \bb S^{p-1}} \mb z \;=\; (\mb I - \mb x \mb x^\top)\mb z.
\end{align*}
Moreover, the oblique manifold $\mc M = \mc {OB}(p,q)$ is a product of $q$ unit spheres, and it is also a smooth manifold embedded in $\bb R^{p \times q}$, where
\begin{align*}
    \mc M = \mc {OB}(p,q) \;=\; \underbrace{\bb S^{p-1} \times \bb S^{p-1} \times \cdots \times \bb S^{p-1} }_{ q\; \text{times} } \;=\; \Brac{ \mb Z \in \bb R^{p \times q} \mid \diag\paren{ \mb Z^\top \mb Z } = \mb 1  }.
\end{align*}
Correspondingly, the tangent space $ \mc P_{\mathrm{T}_{\mb X} \mc {OB}(p, q) }$ is
\begin{align*}
    \mathrm{T}_{\mb X} \mc {OB}(p,q) \;=\; \mr{T}_{\mb x_1} \bb S^{p-1} \times \cdots \times \mr{T}_{\mb x_q} \bb S^{p-1}  \;&=\; \Brac{ \mb Z \in \bb R^{p \times q} \mid  \mb x_i^\top \mb z_i \;=\; 0 ,\; 1\leq i \leq q },\\
     \;&=\; \Brac{ \mb Z \in \bb R^{p \times q} \mid \diag\paren{ \mb X^\top \mb Z } = \mb 0 }
\end{align*}
and the projection operator $\mr{T}_{\mb x_1} \bb S^{p-1}$ is
\begin{align*}
    \mc P_{\mathrm{T}_{\mb X} \mc {OB}(p, q) }(\mb Z) \;&=\; \begin{bmatrix}
    \paren{\mb I - \mb x_1 \mb x_1^\top } \mb z_1 & \cdots & \paren{\mb I - \mb x_q \mb x_q^\top } \mb z_q
    \end{bmatrix}
    \\
    \;&=\;\mb Z - \mb X  \ddiag(\mb X^\top \mb Z).
\end{align*}

\subsection{Derivation of \eqref{eq:euc-hessian} and \eqref{eq:bilinear Hess}} \label{app:hessderivation}
We first derive \eqref{eq:euc-hessian}. Define the curve
\begin{align*}
    \phi(t) :&= f(\mb W + t \mb \Delta_{\mb W}, \mb H + t \mb \Delta_{\mb H}) \\
    &= g(\tau (\mb W + t \mb \Delta_{\mb W})^\top (\mb H + t \mb \Delta_{\mb H})) \\
    &= g(\tau \mb W^\top \mb H + \tau (\mb \Delta_{\mb W}^\top \mb H + \mb W^\top \mb \Delta_{\mb H})t + \tau \mb \Delta_{\mb W}^\top \mb \Delta_{\mb H}t^2) \\
    &= g(\mb M + \mb \delta(t))
\end{align*}
where $\mb M = \tau \mb W^\top \mb H$ and $\mb \delta(t) = \tau (\mb \Delta_{\mb W}^\top \mb H + \mb W^\top \mb \Delta_{\mb H})t + \tau \mb \Delta_{\mb W}^\top \mb \Delta_{\mb H}t^2$ satisfies
\begin{align*}
    \dot{\mb \delta}(t) &= \tau(\mb \Delta_{\mb W}^\top \mb H + \mb W^\top \mb \Delta_{\mb H}) + 2\tau \mb \Delta_{\mb W}^\top \mb \Delta_{\mb H}t  \\
    \ddot{\mb \delta}(t) &= 2 \tau \mb \Delta_{\mb W}^\top \mb \Delta_{\mb H}
\end{align*}
so by chain rule and product rule we have
\begin{align*}
    \dot{\phi}(t) = \left<\dot{\mb \delta}(t), \nabla g(\mb M + \mb \delta(t))\right>
\end{align*}
and
\begin{align*}
    \ddot{\phi}(t) = \left<\ddot{\mb \delta}(t), \nabla g(\mb M + \mb \delta(t))\right> +
    \nabla^2 g(\mb M + \mb \delta(t))[\dot{\mb \delta}(t), \dot{\mb \delta}(t)].
\end{align*}
Then since $\nabla^2 f(\mb W, \mb H)[\mb \Delta, \mb \Delta] = \ddot{\phi}(0)$, we have
\begin{align*}
    &\nabla^2 f(\mb W, \mb H)[\mb \Delta, \mb \Delta] \\
    =& \left<\ddot{\mb \delta}(0), \nabla g(\mb M + \mb \delta(0))\right> +
    \nabla^2 g(\mb M + \mb \delta(0))[\dot{\mb \delta}(0), \dot{\mb \delta}(0)] \\
    =& 2\tau \left<\mb \Delta_{\mb W}^\top \mb \Delta_{\mb H}, \nabla g(\mb M)\right> + \nabla^2 g(\mb M)[\tau(\mb \Delta_{\mb W}^\top \mb H + \mb W^\top \mb \Delta_{\mb H}), \tau(\mb \Delta_{\mb W}^\top \mb H + \mb W^\top \mb \Delta_{\mb H})]
\end{align*}
giving the result. \\~\\
Now we derive \eqref{eq:bilinear Hess}. First, we consider the general case of a function $f$ defined on the oblique manifold $\mc M = \mc {OB}(p,q)$, where $f$ can be smoothly extended to the ambient Euclidean space. We have
\begin{align*}
    \grad f(\mb X) &= \nabla f(\mb X) - \mb X \ddiag(\mb X^\top \nabla f(\mb X)).
\end{align*}
Then
\begin{align*}
    D\grad f(\mb X) [\mb U] &= \lim_{t\rightarrow 0} \mb \Delta(t)
\end{align*}
where
\begin{align*}
    \mb \Delta(t) = & \frac{\grad f(\mb X + t \mb U) - \grad f(\mb X)}{t} \\
    =& \frac{\nabla f(\mb X + t\mb U) - (\mb X + t \mb U) \ddiag((\mb X + t\mb U)^\top \nabla f(\mb X + t \mb U)) - \nabla f(\mb X) + \mb X \ddiag(\mb X^\top \nabla f(\mb X))}{t} \\
    =& \frac{\nabla f(\mb X + t \mb U) - \nabla f(\mb X)}{t} - \mb U \ddiag(\mb X^\top \nabla f(\mb X + t\mb U)) - \mb X \ddiag(\mb U^\top \nabla f(\mb X + t \mb U)) \\
    &- \mb X \ddiag\left(\mb X^\top \frac{\nabla f(\mb X + t \mb U) - \nabla f(\mb X)}{t}\right)
    - t\ \mb U \ddiag(\mb U^\top \nabla f(\mb X + t\mb U))
\end{align*}
so
\begin{align*}
    D\grad f(\mb X) [\mb U] &= \nabla^2 f(\mb X)[\mb U] - \mb U \ddiag(\mb X^\top \nabla f(\mb X)) \\
    &- \mb X \ddiag(\mb U^\top \nabla f(\mb X)) - \mb X \ddiag(\mb U^\top \nabla^2 f(\mb X)[\mb U]).
\end{align*}
Now, for $\mb U \in \mr{T}_{\mb X}\mc M$, we have $\diag(\mb U^\top \mb X) = \mb0$ so
\begin{align*}
    \Hess f(\mb X)[\mb U, \mb U] &= \left<\mb U, \Hess f(\mb X)[\mb U] \right> \\
    &= \left< \mb U, D \grad f(\mb X) - \mb X \ddiag(\mb X^\top D \grad f(\mb X))\right> \\
    &= \left< \mb U, D \grad f(\mb X) \right> \\
    &= \left< \mb U, \nabla^2 f(\mb X)[\mb U] \right> - \left< \mb U, \mb U \ddiag(\mb X^\top \nabla f(\mb X)) \right> \\
    &= \nabla^2 f(\mb X)[\mb U, \mb U] - \left<\mb U \ddiag(\mb X^\top \nabla f(\mb X)), \mb U\right>.
\end{align*}
Now let $f$ be defined as in \eqref{eq:ce-loss func}. Since $(\mb W, \mb H)$ lies on the product manifold $\mc {OB}(d,K) \times \mc {OB}(d,N) = \mc {OB}(d,K+N)$ which is also an oblique manifold, we can simply use the general result above, i.e., for $\mb \Delta \in \mr{T}_{(\mb W, \mb H)}\mc \mc {OB}(d, N+K)$,
\begin{align*}
    \Hess f(\mb W, \mb H)[\mb \Delta, \mb \Delta]=\;& \nabla^2 f(\mb W, \mb H)[\mb \Delta, \mb \Delta] - \left<\mb \Delta_{\mb W} \ddiag(\mb W^\top \nabla_{\mb W} f(\mb W, \mb H)), \mb \Delta_{\mb W}\right> \\
    &- \left<\mb \Delta_{\mb H} \ddiag(\mb H^\top \nabla_{\mb H} f(\mb W, \mb H)), \mb \Delta_{\mb H}\right>
\end{align*}
which gives \eqref{eq:bilinear Hess} after substituting the ordinary Euclidean gradient of $f$.
\subsection{Derivatives of CE Loss}\label{subsec:grad CE}
Note that the CE loss is of the form
\begin{align*}
    \mc L_{\mathrm{CE}}\paren{ \mb z, \mb y_k } \;=\; - \log \paren{  \frac{ \exp(z_k) }{ \sum_{\ell=1}^K \exp(z_{\ell})  } } \;=\; \log\left(\sum_{\ell=1}^K \exp(z_{\ell}) \right) - z_k.
\end{align*}
Then, one can verify
\begin{align*}
   \frac{ \partial \mc L_{\mathrm{CE}}\paren{ \mb z, \mb y_k } }{ \partial z_j } \;=\; \begin{cases}
   \frac{ \exp(z_j) }{ \sum_{\ell=1}^K \exp(z_{\ell})}, & j \neq k, \\
    \frac{ \exp(z_j) }{ \sum_{\ell=1}^K \exp(z_{\ell})} - 1, & j = k,
   \end{cases}
\end{align*}
for all $j \in [K]$. Thus, we have
\begin{align*}
    \nabla \mc L_{\mathrm{CE}}(\mb z,\mb y_k) \;=\; \frac{ \exp\paren{ \mb z }}{\sum_{\ell=1}^K \exp(z_{\ell}) } - \mb e_k\;=\; \eta(\mb z) - \mb e_k,
\end{align*}
where $\eta(\mb z)$ is a softmax function, with
\begin{align*}
    \eta(z_j) \;:=\; \frac{ \exp\paren{ z_j } }{\sum_{\ell=1}^K \exp(z_{\ell}) }.
\end{align*}
Furthermore, we have
\begin{align*}
    \nabla^2 \mc L_{\mathrm{CE}}(\mb z,\mb y_k) \;=\; \diag(\eta(\mb z)) - \eta(\mb z)\eta(\mb z)^\top.
\end{align*}


\section{Proof of \Cref{thm:ce optim} }\label{app:thm global}

In this section, we first simplify Problem \eqref{eq:ce-loss func} by utilizing its structure, then characterize the structure of global solutions of the simplified problem, and finally deduce the struture of global solutions of Problem \eqref{eq:ce-loss func} based on their relationship.
Before we proceed, we can first reformulate Problem \eqref{eq:ce-loss func} as follows. Let 
\begin{align*}
    \mb H \;=\; \begin{bmatrix}
    \mb H^1 & \mb H^2 & \cdots & \mb H^n
    \end{bmatrix} \in \bb R^{d \times N}, \; \mb H^i \;=\; \begin{bmatrix}
    \mb h_{1,i} & \mb h_{2,i} & \cdots & \mb h_{K,i}
    \end{bmatrix} \in \bb R^{d \times K}, \;\forall\ i \in [N],
\end{align*}
and $\bar{f}:\R^{d\times K} \times \R^{d\times K} \rightarrow \R$ be such that
\begin{align}\label{eq:f bar}
    \bar{f}(\mb W, \mb Q) = \frac{1}{K} \sum_{k=1}^K \mc L_{\mathrm{CE}} \paren{ \tau \mb W^\top \mb q_k, \mb y_k }.
\end{align}
Then, we can rewrite the objective function of Problem \eqref{eq:ce-loss func} as
\begin{align}\label{eq:f sum}
    f(\mb W, \mb H) = \frac{1}{n}\sum_{i=1}^n \bar{f}(\mb W, \mb H^i).
\end{align}

\begin{lemma}\label{lem:f bar}
Suppose that $(\mb W^*, \mb Q^*) $ is an optimal solution of 
\begin{align}\label{eq:ce-loss func1}
    \min_{\mb W \in \R^{d\times K}, \mb Q \in \R^{d\times K} }\bar{f}(\mb W, \mb Q)\quad \st\ \mb Q \in \mc {OB}(d,K),\;  \mb W \in\; \mc {OB}(d,K).
\end{align}
Then, $(\mb W^*, \mb H^*)$ with $\mb H^* = \begin{bmatrix}
   \mb Q^* & \mb Q^* & \cdots & \mb Q^*
    \end{bmatrix}$ is an optimal solution of Problem \eqref{eq:ce-loss func}. 
\end{lemma}
\begin{proof}
According to \eqref{eq:f sum}, we note that
\begin{align*}
    & \min\left\{f(\mb W, \mb H):\ \mb H \in \mc {OB}(d,N),\;  \mb W \in\; \mc {OB}(d,K)  \right\} \\
    \ge\ & \frac{1}{n}\sum_{i=1}^n \min\left\{\bar{f}(\mb W^i, \mb H^i):\ \mb H^i \in \mc {OB}(d,K),\;  \mb W^i \in\; \mc {OB}(d,K)  \right\}, 
\end{align*}
where equality holds if $(\mb W^i, \mb H^i) = (\mb W^*, \mb Q^*)$ for all $i \in [n]$ and $(\mb W, \mb H) = (\mb W^*, \mb Q^*)$. This, together with the fact that $(\mb W^*, \mb Q^*)$ is an optimal solution of Problem \eqref{eq:ce-loss func1}, implies the desired result. 
\end{proof}
Based on the above lemma, it suffices to consider the global optimality condition of Problem \eqref{eq:ce-loss func1}. 
\begin{prop}\label{prop:ce opti}
Suppose that the feature dimension is no smaller than the number of classes (i.e., $d \ge K$) and the training labels are balanced in each class (i.e.,  $n = n_1 = \cdots =n_K$). Then, any global minimizer $(\mb W,\mb Q) \in  \mc {OB}(d,K) \times  \mc {OB}(d,K)$ of Problem \eqref{eq:ce-loss func1} satisfies 
\begin{align}\label{rst:ce opti}
    \mb Q = \mb W,\quad \mb Q^{T}\mb Q = \frac{K}{K-1}\left(\mb I_K - \frac{1}{K}\mb 1_K \mb 1_K^\top \right).   
\end{align}
\end{prop}

\begin{proof} 
According to \cite[Lemma D.5]{zhu2021geometric}, it holds for all $k \in [K]$ and any $c_1 > 0$ that
\begin{align*}
    (1+c_1)(K-1)\left( \mc L_{\mathrm{CE}}\left( \tau \mb W^\top \mb q_k, \mb y_k\right) - c_2  \right) & \ge  \tau \left( \sum_{\ell=1}^K \mb w_\ell^\top\mb q_k - K \mb w_k^\top \mb q_k  \right),
\end{align*}
where 
\begin{align*}
    c_2 = \frac{1}{1+c_1}\log\left( (1+c_1)(K-1) \right) + \frac{c_1}{1+c_1}\log\left( \frac{1+c_1}{c_1} \right)    
\end{align*}
and the equality holds when $\mb w_i^\top\mb q_k = \mb w_j^\top\mb q_k$ for all $i,j \neq k$ and 
\begin{align*}
    c_1 = \left( (K-1)\exp\left( \frac{\sum_{\ell=1}^K\mb w_\ell^\top\mb q_k - K\mb w_k^\top\mb q_k}{K-1} \right) \right)^{-1}.
\end{align*}
This, together with \eqref{eq:f bar}, implies
\begin{align*}
    (1+c_1)(K-1)\left(  \bar{f}(\mb W, \mb Q) - c_2 \right) & \ge \frac{\tau}{K}\sum_{k=1}^K\left( \sum_{\ell=1}^K \mb w_\ell^\top\mb q_k - K\mb w_k^\top\mb q_k \right) \\
    & = \frac{\tau}{K}\left( \sum_{k=1}^K\sum_{\ell=1}^K \mb w_k^\top \mb q_\ell - K\sum_{k=1}^K\mb w_k^\top\mb q_k \right) \\
    & = \tau\sum_{k=1}^K\mb w_k^\top\left( \bar{\mb q} - \mb q_k \right),
\end{align*}
where the first inequality becomes equality when $\mb w_i^\top\mb q_k = \mb w_j^\top\mb q_k$ for all $i,j \neq k$ and all $k \in [K]$ and $\bar{\mb q} = \frac{1}{K}\sum_{\ell=1}^K \mb q_\ell$ in the last equality. Note that that $\mb u^\top \mb v \ge -\frac{c_3}{2}\|\mb u\|_2^2 - \frac{1}{2c_3}\|\mb v\|_2^2$ for any $c_3 >0$, where the equality holds when $c_3\mb u = -\mb v$. Consequently, it holds for any $c_3 > 0$ that
\begin{align*}
    (1+c_1)(K-1)\left(  \bar{f}(\mb W, \mb Q) - c_2 \right) & \ge -\tau\sum_{k=1}^K \left( \frac{c_3}{2}\|\mb w_k\|_2^2 + \frac{1}{2c_3}\|\bar{\mb q} - \mb q_k\|_2^2 \right) \\
    & = -\frac{\tau}{2}\left(c_3\sum_{k=1}^K \|\mb w_k\|_2^2 + \frac{1}{c_3}\sum_{k=1}^K \|\mb q_k\|_2^2 - \frac{K}{c_3}\|\bar{\mb q}\|_2^2 \right) \\
    & \ge -\frac{\tau}{2}\left(c_3\sum_{k=1}^K \|\mb w_k\|_2^2 + \frac{1}{c_3}\sum_{k=1}^K \|\mb q_k\|_2^2 \right) = -\frac{\tau}{2}\left(c_3K+\frac{K}{c_3} \right),
\end{align*}
where the first inequality becomes equality when $c_3\mb w_k = \mb q_k - \bar{\mb q}$ for all $k \in [K]$, the second inequality becomes equality when $\bar{\mb q} = \mb 0$, and the last equality is due to $\mb Q \in \mc {OB}(d,K)$ and $\mb W \in\; \mc {OB}(d,K)$. Thus, we have
\begin{align*}
    (1+c_1)(K-1)\left(  \bar{f}(\mb W, \mb Q) - c_2 \right) \ge -\frac{\tau K}{2}\left(c_3 + \frac{1}{c_3} \right),
\end{align*}
where the equality holds when $\mb w_i^\top\mb q_k = \mb w_j^\top\mb q_k$ for all $i,j \neq k$ and all $k \in [K]$, $c_3\mb w_k = \mb q_k$ for all $k \in [K]$, and $\sum_{k=1}^K \mb q_k = \mb 0$.
This, together with $\mb Q \in \mc {OB}(d,K)$ and $\mb W \in\; \mc {OB}(d,K)$, implies $c_3 = 1$. Thus, we have $ \mb q_k = \mb w_k$ for all $k \in [K]$ and
\begin{align*}
    \bar{f}(\mb W, \mb Q)  \ge -\frac{\tau K}{(1+c_1)(K-1)} + c_2.
\end{align*}
This further implies that $\sum_{k=1}^K \mb w_k = 0$, $\mb w_i^\top\mb w_k = \mb w_j^\top\mb w_k$ for all $i,j \neq k$ and all $k \in [K]$. Then, it holds that for all $1 \le k \neq \ell \le K$ that
\begin{align*}
    \langle \mb w_\ell, \mb w_k \rangle = -\frac{1}{K-1}.
\end{align*}
These, together with $\mb Q \in \mc {OB}(d,K)$ and $\mb W \in\; \mc {OB}(d,K)$, imply \eqref{rst:ce opti}. 
\end{proof}

\begin{proof}[Proof of \Cref{thm:ce optim}]
According to \eqref{eq:f sum}, \Cref{lem:f bar}, and \Cref{prop:ce opti}, the global solutions of Problem \eqref{eq:ce-loss func} take the form of 
\begin{align*}
    \mb h_{k,i} = \mb q_k,\ \mb w_k = \mb q_k,\ \forall\ k \in [K],\ i \in [N], 
\end{align*}
and
\begin{align*}
     \mb Q^{T}\mb Q = \frac{K}{K-1}\left(\mb I_K - \frac{1}{K}\mb 1_K \mb 1_K^\top \right).
\end{align*}
Based on this and the objective function in Problem \eqref{eq:ce-loss func}, the value at an optimal solution $(\mb W^*, \mb H^*)$ is 
\begin{align*}
    f(\mb W^*, \mb H^*) =\log\left( 1 + \frac{(K-1)\exp\left( -\frac{\tau}{K-1} \right)}{\exp(\tau)}\right) = \log\left( 1 + (K-1)\exp\left( -\frac{K\tau}{K-1} \right)\right).
\end{align*}
Then, we complete the proof.
\end{proof}

\section{Proof of \Cref{thm:ce landscape}}\label{app:thm landscape}

In this section, we first analyze the first-order optimality condition of Problem \eqref{eq:ce-loss func}, then characterize the global optimality condition of Problem \eqref{eq:ce-loss func}, and finally prove no spurious local minima and strict saddle point property based on the previous optimality conditions. 
For ease of exposition, let us denote
\begin{align}\label{eq:M grad}
    \mb M := \tau\mb W^\top\mb H,\ g(\mb M) := f(\mb W, \mb H) = \frac{1}{N}  \sum_{i=1}^n \sum_{k=1}^K \mc L_{\mathrm{CE}}( \mb m_{k,i}, \mb y_k). 
\end{align}
Then we have the gradient
\begin{align*}
    \nabla f(\mb W,\mb H) \;=\; \paren{ 
    \nabla_{\mb W} f(\mb W,\mb H), \nabla_{\mb H} f(\mb W,\mb H) }
\end{align*}
with 
\begin{align}\label{grad:f}
    \nabla_{\mb W} f(\mb W,\mb H) \;=\; \tau \mb H \nabla g(\mb M)^\top,\ \nabla_{\mb H} f(\mb W,\mb H) \;=\; \tau \mb W \nabla g(\mb M), 
\end{align}
and
\begin{align}\label{eqn:gradient-CE}
    \nabla g(\mb M) \;=\; \begin{bmatrix} \eta(\mb m_{1,1}) &\cdots & \eta(\mb m_{K,n}) \end{bmatrix} - \mb I_K \otimes \mb 1_n^\top, \quad \eta(\mb m) \;=\; \frac{\exp\paren{\mb m}}{ \sum_{i=1}^K \exp\paren{m_i }  }.
\end{align}

\subsection{First-Order Optimality Condition}

By using the tools in \Cref{subsec:Rie tools}, we can calculate the Riemannian gradient at a given point  $(\mb W, \mb H) \in \mc {OB}(d,N) \times \mc {OB}(d,K)$ as in \eqref{eq:rigrad H} and \eqref{eq:rigrad W}. 
Thus, for a point $(\mb W, \mb H) \in \mc {OB}(d,N) \times \mc {OB}(d,K)$, the first-order optimality condition of Problem \eqref{eq:ce-loss func} is 
\begin{align}
   \grad_{\mb W } f(\mb W,\mb H) \;&=\;  \tau \mb W \nabla g(\mb M) - \tau\mb H  \ddiag \paren{ \mb H^\top \mb W \nabla g(\mb M)} \;=\; \mb 0, 
    \label{FONC:W}\\
    \grad_{\mb H} f(\mb W,\mb H) \;&=\; \tau \mb H \nabla g(\mb M)^\top  - \tau \mb W \ddiag \paren{ \mb W^\top \mb H \nabla g(\mb W)^\top } \;=\; \mb 0. \label{FONC:H}
\end{align}
We denote the set of all critical points  by
\begin{align*}
    \mc C \;:=\; \Brac{ (\mb W,\mb H) \in \mc {OB}(d,K) \times \mc {OB}(d,N) \; \mid\;  \grad_{\mb H} f(\mb W,\mb H) = \mb 0, \; \grad_{\mb W} f(\mb W,\mb H) = \mb 0 }.
\end{align*}

\begin{lemma}\label{lem:alpha beta}
Suppose that $\mb g_i \in \R^K$ and $\mb g^k \in \R^N$ denote the $i$-th column and $k$-th row vectors of the matrix
\begin{align*}
     \mb G : = \nabla g (\mb M) \in \bb R^{K \times N},
\end{align*}
respectively. Let $\mb \alpha \in \R^K$ and $\mb \beta \in \R^N$ be such that 
\begin{align}\label{def:alpha beta}
     \alpha_k \;=\; \innerprod{\mb w_k}{\mb H \mb g^k},\forall\ k \in [K],\quad  \beta_i \;=\; \innerprod{\mb h_i}{\mb W \mb g_i},\forall\ i \in [N].
\end{align}
Then it holds for any $(\mb W, \mb H) \in \mc C$ that
\begin{align}\label{eq:FONC}
    \mb H \mb g^k \;=\; \alpha_k \mb w_k,\ \forall\ k \in [K],\quad \mb W \mb g_i \;=\; \beta_i \mb h_i,\ \forall\ i \in [N].
\end{align}
and
\begin{align}\label{eq:alpha-beta}
|\alpha_k| = \|\mb H \mb g^k\|_2,\ k = 1,\dots K,\quad   |\beta_i| = \|\mb W \mb g_i\|_2,\ i=1,\dots,N.
\end{align}

\end{lemma}
\begin{proof}
According to \eqref{grad:f}, we have 
\begin{align*}
     \mb H \mb G^\top = \begin{bmatrix} 
    \mb H\mb g^1 & \dots & \mb H\mb g^K 
\end{bmatrix},\quad \mb W \mb G = \begin{bmatrix} 
    \mb W\mb g_1 & \dots & \mb W\mb g_K
\end{bmatrix}
\end{align*}
Using \eqref{def:alpha beta}, we can compute
\begin{align*}
    \ddiag\left(\mb W^\top\mb H \mb G^\top\right) = \diag( \mb \alpha),\quad  \ddiag\left(\mb H^\top\mb W \mb G \right) = \diag( \mb \beta )
\end{align*}
This, together with \eqref{FONC:W} and \eqref{FONC:H}, implies \eqref{eq:FONC}.
Since $\|\mb w_k\|_2 = 1$ for all $k \in [K]$ and $\|\mb h_i\|_2 = 1$ for all $i \in [N]$, by
\begin{align*}
   \alpha_k^2 =  \langle \alpha_k\mb w_k, {\mb H \mb g^k} \rangle = \left\|\mb H\mb g^k \right\|_2^2,\quad \beta_i^2 =  \langle \beta_i\mb h_i, \mb W \mb g_i \rangle = \left\|\mb W\mb g_i \right\|_2^2
\end{align*}
which implies \eqref{eq:alpha-beta}.
\end{proof}

\subsection{Characterization of Global Optimality} 

According to \Cref{thm:ce optim}, it holds that for any global solution $(\mb W, \mb H) \in \mc {OB}(d,N) \times \mc {OB}(d,K)$ that 
\begin{align}\label{eq:opti sol}
    \mb H = \mb W \otimes \mb 1_n^\top,\ \mb W^\top \mb W = \frac{K}{K-1}\left( \mb I_K - \frac{1}{K}\mb 1_K \mb 1_K^\top\right),
\end{align}
where $\otimes$ denotes the Kronecker product.

\begin{lemma}\label{lem:global-saddle}
Given any critical point $(\mb W,\mb H) \in \mc C$, let $\mb \alpha \in \R^K$ and $\mb \beta \in \R^N$ be defined as in \eqref{def:alpha beta}. Then, $(\mb W,\mb H)$ is a global solution of Problem \eqref{eq:ce-loss func} if and only if the corresponding $(\mb \alpha, \mb \beta)$ satisfies  
\begin{align}\label{global:alpha beta}
    \alpha_k \le -\sqrt{n}\|\nabla g(\mb M)\|,\ \forall\ k \in [K],\quad \beta_i \le -\frac{\|\nabla g(\mb M)\|}{\sqrt{n}},\ \forall\ i \in [N],
\end{align}
where $\mb M = \tau\mb W^\top\mb H$.
\end{lemma}
\begin{proof}
Suppose that  $(\mb W,\mb H) \in \mc C$ is an optimal solution. According to \eqref{eq:opti sol}, one can verify that
\begin{align*}
    \mb W^\top \mb H = \mb W^\top \left(\mb W \otimes \mb 1_n^\top \right) = \frac{K}{K-1}\left(\mb I_K  - \frac{1}{K}\mb 1_K\mb 1_K^\top \right) \otimes \mb 1_n^\top.
\end{align*}
According to this and \eqref{eq:M grad}, we can compute
\begin{align}\label{eq:grad G}
    \nabla g(\mb M) = \frac{-K\exp\left(-\frac{1}{K-1}\right)}{\exp(1)+(K-1)\exp\left(-\frac{1}{K-1}\right)}\left(\mb I_K  - \frac{1}{K}\mb 1_K\mb 1_K^\top\right) \otimes \mb 1_n^\top. 
\end{align}
This, together with $\alpha_k = \langle \mb w_k,\mb H \mb g^k \rangle$, yields for all $k \in K$,
\begin{align}\label{alpha:k}
\alpha_k = \langle \mb H^\top \mb w_k, \mb g^k \rangle \;=\; \frac{-nK\exp\left(-\frac{1}{K-1}\right)}{\exp(1)+(K-1)\exp\left(-\frac{1}{K-1}\right)}. 
\end{align}
By the same argument, we can compute for all $i \in [N]$,
\begin{align}\label{beta:k}
\beta_i =  \frac{-K\exp\left(-\frac{1}{K-1}\right)}{\exp(1)+(K-1)\exp\left(-\frac{1}{K-1}\right)}. 
\end{align}
According to \eqref{eq:grad G}, one can verify 
\begin{align*}
\|\nabla g(\mb M)\| = \frac{\sqrt{n}K\exp\left(-\frac{1}{K-1}\right)}{\exp(1)+(K-1)\exp\left(-\frac{1}{K-1}\right)}. 
\end{align*}
This, together with \eqref{alpha:k} and \eqref{beta:k}, implies \eqref{global:alpha beta} \\~\\
Suppose that a critical point $(\mb W^*,\mb H^*) \in \mc C$ satisfies \eqref{global:alpha beta}. Let $\mb M^* = \tau\mb W^{*^\top}\mb H^*$  and $\lambda = \|\nabla g(\mb M^*)\|$. 
According to \eqref{eq:alpha-beta} and the fact that $\|\mb w_k^*\|=1$ and $\|\mb h_k^*\|=1$ for all $k=1,\dots K$, we have
\begin{align*}
    & \sum_{k=1}^K \alpha_k^{*^2} = \|\mb H^* \nabla g(\mb M^*)^\top\|_F^2 \le  \|\nabla g(\mb M^*)\|^2\|\mb H^*\|_F^2 = \lambda^2 N,\\
    & \sum_{i=1}^N \beta_i^{*^2} = \|\mb W^* \nabla g(\mb M^*)\|_F^2 \le  \|\nabla g(\mb M^*)\|^2\|\mb W^*\|_F^2 = \lambda^2 K. 
\end{align*}
This, together with \eqref{global:alpha beta}, implies
\begin{align}\label{global:alpha=beta}
    \alpha_k^* = -\sqrt{n}\lambda,\ \forall\ k \in [K],\quad  \beta_i^* = -\frac{\lambda}{\sqrt{n}},\ \forall\ i \in [N].
\end{align}
Then, we consider the following regularized problem:
\begin{align}\label{P:Re CE}
    \min_{\mb W \in \R^{d\times K}, \mb H \in \R^{d\times N}} f(\mb W, \mb H) + \frac{\lambda\sqrt{n}}{2}\|\mb W\|_F^2 + \frac{\lambda}{2\sqrt{n}}\|\mb H\|_F^2. 
\end{align}
According to the fact that $(\mb W^*,\mb H^*)$ is a critical point of Problem \eqref{eq:ce-loss func} and satisfies \eqref{global:alpha=beta}, \eqref{FONC:W}, and \eqref{FONC:H}, we have
\begin{align}
    \begin{cases}
    \mb H^* \nabla g(\mb M^*)^\top + \lambda{\sqrt{n}} \mb W^* = \mb 0, \\
    \mb W^* \nabla g(\mb M^*)  + \lambda \mb H^*/\sqrt{n} = \mb 0.
    \end{cases}
\end{align}
This, together with the first-order optimality condition of Problem \eqref{P:Re CE}, yields that $(\mb W^*,\mb H^*)$ is a critical point of Problem \eqref{P:Re CE}. According to  \cite[Lemma C.4]{qu2020geometric} and $\|\nabla g(\mb M^*)\| = \lambda$, it holds that $(\mb W^*,\mb H^*)$ is an optimal solution of Problem \eqref{P:Re CE}. This, together with \cite[Theorem 3.1]{qu2020geometric}, yields that $(\mb W^*,\mb H^*) \in \mc C$ satisfies
\begin{align*}
    \mb H^* = \mb W^* \otimes \mb 1_n^\top,\ \mb W^{*^\top}\mb W^* =  \frac{K}{K-1}\left(\mb I_K  - \frac{1}{K}\mb 1_K\mb 1_K^\top\right). 
\end{align*}
According to \Cref{thm:ce optim}, we conclude that $(\mb W^*,\mb H^*)$ is an optimal solution of Problem \eqref{eq:ce-loss func}. Then, we complete the proof.  
\end{proof}

\subsection{Negative Curvature at Saddle Points}

\begin{lemma}\label{lem:alphabetasum}
Let $\mb \alpha$ and $\mb \beta$ be defined as in \Cref{lem:alpha beta}. Then $\sum_{k=1}^K \alpha_k = \sum_{i=1}^N \beta_i$.
\end{lemma}

\begin{proof}
Given the definition of $\mb \alpha$ and $\mb \beta$ in \eqref{def:alpha beta}, this follows directly from cyclic property of trace:
\begin{align*}
    \sum_{k=1}^K \alpha_k = \trace(\mb W^\top \mb H \mb G^\top) = \trace(\mb G \mb H^\top \mb W) = \trace(\mb H^\top \mb W \mb G) = \sum_{i=1}^{N} \beta_i,
\end{align*}
as desired.
\end{proof}

\begin{lemma}\label{lem:beta zero}
Suppose $(\mb W, \mb H)$ is a critical point and there exists $i \in [N]$ such that $\beta_i = 0$. Then there exists $\mb w \in \bb S^{d-1}$ such that $\mb W = \mb w \mb 1_K^\top$. Furthermore, we have $\beta_1 = \ldots = \beta_N = 0$.
\end{lemma}

\begin{proof}
Suppose $nk \leq i < n(k+1)$ for $k \in [K]$ (i.e., $\mb h_i$ has label $y_k$). Thus, we can write each entry of the gradient $\mb g_i$ of the CE loss as 
\begin{align*}
    g_{i\ell} = \begin{cases}
    p_{ik}-1 & \ell = k \\
    p_{i\ell} & \ell \neq k
    \end{cases} \quad \mbox{ where } \quad p_{i\ell} = \frac{\exp(\tau \mb w_{\ell}^\top \mb h_i)}{\sum_{j=1}^K \exp(\tau \mb w_j^\top\mb h_i)}.
\end{align*}
Since $\exp(\cdot) > 0$ and $K \geq 2$, we have $0 < p_{i\ell} < 1$. Given that $\beta_i=0$ and $\norm{\mb h_i}{2}=1$, from \eqref{eq:FONC} we know that we must have $\mb W \mb g_i = \mb 0$, which further gives
\begin{align*}
    g_{ik} \mb w_k + \sum_{\ell \neq k} g_{i\ell} \mb w_{\ell} = 0.
\end{align*}
Given $1-p_{ik}>0 $, equivalently we have
\begin{align*}
    \mb w_k = \sum_{\ell \neq k} \frac{p_{i\ell}}{1-p_{ik}} \mb w_{\ell},
\end{align*}
where $\sum_{\ell \neq k} \frac{p_{i\ell}}{1-p_{ik}} = 1$ and $p_{i\ell} > 0$ so $\mb w_k$ is a strict convex combination of points $\{\mb w_{\ell}\}_{\ell \neq k}$ on the unit sphere. But since $\mb w_k$ also lies on the unit sphere, and the convex hull of points on the sphere only intersects with the sphere at $\{\mb w_{\ell}\}_{\ell \neq k}$, we must have all $\mb w_{\ell}$ be identical, i.e., $\mb w_1 = \ldots = \mb w_K$. Therefore, we can write $\mb W = \mb w_1 \mb 1_K^\top$, and consequently
\begin{align*}
    \mb W \mb G = \mb w_1 \mb 1_K^\top \mb G = \mb 0,
\end{align*}
where the last equality follows from the fact that $\mb 1_K^\top \mb G = \mb 1_K^\top \nabla \mb g(\mb M) = \mb 0$. Thus, given $\beta_i = \innerprod{\mb h_i}{\mb W \mb g_i}$, from the above we have $\beta_1 = \ldots = \beta_N = 0$.
\end{proof}

\begin{lemma}\label{lem:common a}
For any $\mb H \in \mc {OB}(d, N)$ and $\mb w \in \bb S^{d-1}$, there exists at least one $\mb a \in \bb S^{d-1}$ such that for any $0 < \tau < 2(d-2)(1 + (\Kmod)/K)^{-1}$, we have
\begin{align}\label{eq:common a}
    \mb a^\top \mb w = \mb 0 \quad \mbox{ and } \quad \|\mb H^\top \mb a\|_2^2 < \Gamma := \frac{2N}{\tau(1+(\Kmod)/K)+2}.
\end{align}
\end{lemma}

\begin{proof}
To establish the result, we need to show that there exists a linear subspace $\mc S \subset \bb R^d$ with $\dim (\mc S) \geq 2$ such that for any nonzero $\mb z \in \mc S$ we have $\|\mb H^\top \mb z\|_2^2 < \Gamma \|\mb z\|_2^2$. Then 
\begin{align*}
    \dim(\mc S \cap \mathcal{N}(\mb w)) > 0,
\end{align*}
where $\mc N(\mb w)$ denotes the null space of $\mb w$, so if we choose unit-norm $\mb a \in \mc S \cap \mathcal{N}(\mb w)$, we can obtain the desired results.
Let $(\sigma_{\ell}^2(\mb H), v_{\ell})$ denote the $\ell$-th eigenvalue-eigenvector pair of $\mb H \mb H^\top \in \bb R^{d\times d}$ for $\ell \in [d]$. Given the fact $\mb H \in \mc {OB}(d, N)$, it is obvious that 
\begin{align*}
    \|\mb H\|_F^2 = \sum_{\ell=1}^d \sigma_{\ell}^2(\mb H) = \sum_{j=1}^N \norm{\mb h_j}{2}^2 = N.
\end{align*}
Now suppose that $\sigma^2_{d-1}(\mb H) \geq \Gamma$. Then we must have
\begin{align*}
    N = \sum_{i=1}^d \sigma_i^2(\mb H) \geq (d-1)\Gamma = (d-1)\frac{2N}{\tau(1+(\Kmod)/K)+2}
\end{align*}
which implies $\tau \geq 2(d-2)(1 + (\Kmod)/K)^{-1}$, but this contradicts the assumption on $\tau$. Therefore $\sigma^2_{d-1}(\mb H) < \Gamma$, so we can choose $\mc S = \mbox{span}(\{v_{d-1}, v_d\})$, which suffices to give the result by the above argument.
\end{proof}

\noindent We are now ready to show that at any critical point that is not globally optimal, we can find a direction along which the Riemannian Hessian has a strictly negative curvature at this point. \\~\\
Recall $\mb M := \tau \mb W^\top \mb H$ and $\mb G := \nabla g(\mb M)$, as well as the definition of $\mb \alpha \in \R^K,\ \mb \beta \in \R^N$ in \eqref{def:alpha beta}. As mentioned at the beginning of \Cref{app:thm global}, we can write $\mb H$ as
\begin{align*}
    \mb H \;=\; \begin{bmatrix}
    \mb H^1 & \mb H^2 & \cdots & \mb H^n
    \end{bmatrix} \in \bb R^{d \times N}, \; \mb H^i \;=\; \begin{bmatrix}
    \mb h_{1,i} & \mb h_{2,i} & \cdots & \mb h_{K,i}
    \end{bmatrix} \in \bb R^{d \times K}, \;\forall\ i \in [N].
\end{align*}
As a final remark, the bilinear form of the Riemannian Hessian in \eqref{eq:bilinear Hess} can be written as
\begin{align}\label{eq:bilinear Hess alpha beta}
    \Hess f(\mb W,\mb H)[\mb \Delta,\mb \Delta] = \nabla^2 f(\mb W,\mb H)[ \mb \Delta,\mb \Delta ] - \tau\sum_{k=1}^K \alpha_k \|\mb \delta_{W_k}\|_2^2 - \tau\sum_{i=1}^N \beta_i \|\mb \delta_{H_i}\|_2^2
\end{align}
where $\nabla^2 f(\mb W,\mb H)[ \mb \Delta,\mb \Delta ]$ is given in \eqref{eq:euc-hessian}, and $\mb \delta_{W_k}$, $\mb \delta_{H_i}$ are the $k$-th and $i$-th columns of $\mb \Delta_{\mb W}$ and $\mb \Delta_{\mb H}$ respectively.

\begin{prop}\label{prop:negative curva}
Suppose $d > K$ and $ \tau < 2(d-2)(1+(\Kmod)/K)^{-1}$. For any critical point $(\mb W, \mb H) \in \mc C$ that is not globally optimal, there exists $\mb \Delta = \paren{ \mb \Delta_{\mb W}, \mb \Delta_{\mb H} } \in \mathrm{T}_{ \mb W }  \mc {OB}(d,K )  \times \mathrm{T}_{ \mb H}\mc {OB} (d,N)$ such that
\begin{align}\label{eqn:negative-Hessian}
    \Hess f(\mb W,\mb H)[\mb \Delta,\mb \Delta] \;<\; 0.
\end{align}
\end{prop}
\begin{proof} We proceed by considering two separate cases for the value of $\mb \beta$: $\beta_i = 0$ for some $i\in [N]$, and $\beta_i\neq 0$ for all $i\in [N]$. \\~\\
\noindent \textbf{Case 1:} Suppose $\beta_i = 0$ for some $i \in [N]$. In this case, by \Cref{lem:beta zero}, we know that $\mb W = \mb w \mb 1_K^\top$ for some $\mb w \in \bb S^{d-1}$ and $\mb \beta = \mb 0$.
We have that $\mb M = \tau \mb 1_K \mb w^\top \mb H$, and so 
\begin{align}\label{eqn:G-A}
    \mb G \;=\; -\frac{1}{N} \begin{bmatrix}
    \mb A & \cdots & \mb A
    \end{bmatrix}\in \bb R^{K \times N}, \quad \mb A \;=\;  \mb I_K - \frac{1}{K}\mb 1_K \mb 1_K^\top \in \bb R^{K\times K}. 
\end{align}
For the the $i$-th column of $\mb M$, i.e. $\mb m_i$, we have the Hessian 
\begin{align}\label{eqn:hessian-lce}
    \nabla^2 \mc L_{\mathrm{CE}}(\mb m_i,\mb y_k) = \frac{1}{K} \mb I_K - \frac{1}{K^2}\mb 1_K \mb 1_K^\top =\frac{1}{K} \mb A
\end{align}
Using \Cref{lem:common a}, choose $\mb a \in \bb S^{d-1}$ satisfying \eqref{eq:common a}. Additionally, choose a vector $\mb u \in \bb R^K$ with each entry $u_k = (-1)^{k+1} $ (noting that $\sum_k u_k = K \mbox{ mod } 2$). Now, we construct the negative curvature direction $\mb \Delta = (\mb \Delta_{\mb W}, \mb \Delta_{\mb H})$ as
\begin{align*}
    \mb \Delta_{\mb W} = \mb a \mb u^\top, \;\mb \Delta_{\mb H} = \begin{bmatrix} \mb \Delta_{\mb H^1} & \cdots \mb \Delta_{\mb H^n} \end{bmatrix}
\end{align*}
where 
\begin{align*}
    \mb \Delta_{\mb H^i} = \mb a \mb u^\top - \mb H^i \ddiag(\mb H^{i\top} \mb a \mb u^\top), \; \forall i \in [n].
\end{align*}
First, let $\mb \delta_{M_i}$ denote the $i$-th column of $\mb \Delta_{\mb M} := \mW^\top \mDelta_{\mH} + \mDelta_{\mW}^\top \mH$, so that 
\begin{align}\label{eqn:delta-i}
    \mb \delta_{M_i} = (\mb w^\top \mb \delta_{H_i})\mb 1_K + (\mb h_i^\top \mb a) \mb u.
\end{align}
Then from \eqref{eq:M grad} and \eqref{eqn:hessian-lce}, we know that 
\begin{align*}
     \nabla^2g(\mb W^\top \mb H) \brac{  \tau \mb \Delta_{\mb M} ,\tau \mb \Delta_{\mb M} } = \frac{\tau^2}{NK}\sum_{i=1}^N \mb \delta_{M_i}^\top \mb A \mb \delta_{M_i}.
\end{align*}
Since $\mb A \mb 1_K = \mb 0$ and $\mb u^\top \mb A \mb u = K - (\Kmod) / K$, by \eqref{eqn:delta-i} we have
\begin{align*}
     \nabla^2g(\mb W^\top \mb H) \brac{  \tau \mb \Delta_{\mb M} ,\tau \mb \Delta_{\mb M} } &= \frac{\tau^2}{NK}
    \left(K-\frac{\Kmod}{K}\right)\sum_{i=1}^N (\mb h_i^\top \mb a)^2 \\
    &= \frac{\tau^2}{NK}\left(K-\frac{\Kmod}{K}\right)\|\mb H^\top \mb a\|_2^2.
\end{align*}
On the other hand, by \eqref{eqn:G-A} we have 
\begin{align*}
    2\tau\left<\mb G, \mb \Delta_{\mb W}^\top \mb \Delta_{\mb H}\right> &= -\frac{2\tau}{N}\sum_{i=1}^n \trace(\mb A \mb \Delta_{\mb W}^\top \mb \Delta_{\mb H^i}) \\
    &= -\frac{2\tau}{N} \sum_{i=1}^n \trace\paren{\mb A\mb u \mb u^\top \diag \paren{ 1 - (\mb h_{1, i}^\top \mb a)^2,\ldots, 1 - (\mb h_{K, i}^\top \mb a)^2} } \\
    &= -\frac{2\tau}{N} \sum_{i=1}^n \mb u^\top \diag \paren{ 1 - (\mb h_{1, i}^\top \mb a)^2,\ldots, 1 - (\mb h_{K, i}^\top \mb a)^2} \left(\mb u - \frac{\Kmod}{K} \mb 1_K \right) \\
    &= -\frac{2\tau}{N} \sum_{i=1}^n \sum_{k=1}^K (1 - (\mb h_{k, i}^\top \mb a)^2)u_k^2 + (\Kmod)\frac{2\tau}{NK} \sum_{i=1}^n \sum_{k=1}^K (1 - (\mb h_{k, i}^\top \mb a)^2) u_k \\
    &\leq -\frac{2\tau }{N}\paren{ N-\|\mb H^\top \mb a\|_2^2} + (\Kmod)\frac{2\tau}{NK} \paren{ N - \|\mb H^\top \mb a\|_2^2} \\
    &= -\frac{2\tau}{NK} \paren{ N-\|\mb H^\top \mb a\|_2^2} \paren{ K - (\Kmod)}.
\end{align*}
Finally, the remaining term $- \tau\sum_{k=1}^K \alpha_k \|\mb \delta_{W_k}\|_2^2 - \tau\sum_{i=1}^N \beta_i \|\mb \delta_{H_i}\|_2^2$ in \eqref{eq:bilinear Hess alpha beta} vanishes, which is due to the fact that $\mb \beta = \mb 0$ and 
\begin{align*}
    \sum_{k=1}^K \alpha_k \|\mb \delta_{W_k}\|_2^2 = \sum_{k=1}^K \alpha_k u_k^2 = \sum_{k=1}^K \alpha_k = 0,
\end{align*}
where the last equality follows by \Cref{lem:alphabetasum} that $\sum_{k=1}^K \alpha_k = \sum_{i=1}^N \beta_i =0 $. Therefore, plugging both bounds above into \eqref{eq:bilinear Hess alpha beta}, we obtain 
\begin{align*}
    &\Hess f(\mb W,\mb H)[\mb \Delta,\mb \Delta] \\
    \leq\;& \frac{\tau^2}{NK}\left(K-\frac{\Kmod}{K}\right)\|\mb H^\top \mb a\|_2^2 -\frac{2\tau}{NK}(N-\|\mb H^\top \mb a\|_2^2)(K - (\Kmod)) \\
    =\;& \frac{\tau(K-(\Kmod))}{NK} \left(\tau\left[\frac{K^2-(\Kmod)}{K(K-(\Kmod))}\right] \|\mb H^\top \mb a\|_2^2 - 2 (N-\|\mb H^\top \mb a\|_2^2)\right)\\
    =\;& \frac{\tau(K-(\Kmod))}{NK} \left[ \left(\tau [1 + (\Kmod)/K)] + 2\right) \|\mb H^\top \mb a\|_2^2 - 2N\right] < 0,
\end{align*}
where the last inequality follows by our choice of $\mb a \in \bb S^{d-1}$ in \Cref{lem:common a}. Thus we obtain the desired result in \eqref{eqn:negative-Hessian} for this case. \\~\\
\textbf{Case 2:} Suppose $\beta_i \neq 0$ for all $i \in [N]$. Using the fact that $d > K$, choose $\mb a \in \bb S^{d-1}$ such that $\mb W^\top \mb a = \mb 0$. By \Cref{lem:alpha beta}, given that $\mb W \mb g_i = \beta_i \mb h_i$ for all $i \in [N]$, we have
\begin{align*}
    \mb a^\top \mb W \mb g_i = \beta_i \mb a^\top \mb h_i = 0,\quad \forall \; i \in [N].
\end{align*}
Thus, as $\beta_i \neq 0$ for all $i \in [N]$, this simply implies that  $\mb H^\top \mb a = \mb 0$. Now using \Cref{lem:global-saddle}, for any non-optimal critical point $(\mb W, \mb H)$, there exists at least one $k \in [K]$ or $i \in [N]$ such that either
\begin{align}\label{eqn:alpha-beta-inequal}
    \alpha_k > -\sqrt{n}\|\mb G\|,\quad \text{or}\quad  \beta_i > - {\|\mb G\|}/{\sqrt{n}}.
\end{align}
Let $\mb u_1 \in \R^K$ and $\mb v_1 \in \R^N$ be the left and right unit singular vectors associated with the leading singular values of $\mb G$, respectively. In other words, we have  
\begin{align}\label{eq:svd G}
    \mb u_1^\top \mb G \mb v_1 = \|\mb G\|.
\end{align}
By letting $\mb u = - \mb u_1/\sqrt[4]{n},\ \mb v = \sqrt[4]{n} \mb v_1$, we construct the negative curvature direction as 
\begin{align}
    \mb \Delta = \left( \mb \Delta_{\mb W},  \mb \Delta_{\mb H}\right) = \left(\mb a\mb u^\top ,\ \mb a \mb v^\top \right).
\end{align}
Since $\mb W^\top \mb a = \mb 0,\mb H^\top \mb a = \mb 0$, we have
\begin{align*}
\mb W^\top \mb \Delta_{\mb H} + \mb \Delta_{\mb W}^\top \mb H \;=\; \mb W^\top \mb a\mb v^\top +  \mb u\mb a^\top \mb H \;=\; \mb 0,
\end{align*}
so that from \eqref{eq:euc-hessian} we have  
\begin{align*}
    \nabla^2 f(\mb W,\mb H)[ \mb \Delta,\mb \Delta ] \;=\;&  \nabla^2g(\mb M) \brac{ \tau\left(\mW^\top \mDelta_{\mH} + \mDelta_{\mW}^\top \mH \right), \tau\left(\mW^\top \mDelta_{\mH} + \mDelta_{\mW}^\top \mH\right) } \nonumber \\
    &+ 2\tau \innerprod{\mb G}{\mDelta_{\mW}^\top \mDelta_{\mH}}.
\end{align*}
Thus, from \eqref{eq:bilinear Hess alpha beta}, combining all the above derivations we obtain  
\begin{align*}
    \Hess f(\mb W, \mb H)[\mb \Delta,\mb \Delta] & \;=\; 2\tau \innerprod{\mb G}{\mDelta_{\mW}^\top \mDelta_{\mH}} - \tau \sum_{k=1}^K \alpha_k \norm{ \mb \delta_{W_k} }{2}^2 - \tau \sum_{i=1}^N \beta_i \norm{ \mb \delta_{H_i} }{2}^2. \\
    &\;=\; -2\tau \innerprod{\mb G}{\mb u_1 \mb v_1^\top } - \tau \left( \sum_{k=1}^K \frac{\alpha_ku_{1,k}^2}{\sqrt{n}} + \sum_{i=1}^N \sqrt{n}\beta_i v^2_{1,i}\right) \\
    &\;=\; \tau\left(-2\|\mb G\| - \sum_{k=1}^K \frac{\alpha_ku_{1,k}^2}{\sqrt{n}} - \sum_{i=1}^N \sqrt{n}\beta_i v^2_{1,i}\right)
\end{align*}
where the last equality follows from \eqref{eq:svd G}. On the other hand, by \Cref{lem:global-saddle}, the fact we derived in \eqref{eqn:alpha-beta-inequal} that there exists $k \in [K]$ such that $\alpha_k > -\sqrt{n}\|\mb G\|$ or there exists $i \in [N]$ such that $\beta_i > - {\|\mb G\|}/{\sqrt{n}}$, and that $\|\mb u_1\|_2=\|\mb v_1\|_2=1$, we obtain
\begin{align*}
     - \sum_{k=1}^K \frac{\alpha_ku_{1,k}^2}{\sqrt{n}} - \sum_{i=1}^N \sqrt{n}\beta_i v^2_{1,i} \;< \; \norm{\mb G}{} \paren{ \sum_{k=1}^K u_{1,k}^2 +   \sum_{i=1}^N v_{1,i}^2  } \;=\; 2 \norm{\mb G}{}.
\end{align*}
Therefore, we have
\begin{align*}
    \Hess f(\mb W, \mb H)[\mb \Delta,\mb \Delta] \;<\;  \tau\left(-2\|\mb G\| + 2\|\mb G\|\right) \;=\; 0,
\end{align*}
as desired.
\end{proof}

\begin{proof}[Proof of \Cref{thm:ce landscape}]
Let $(\mb W,\mb H) \in \mc {OB}(d,K) \times \mc {OB}(d,N)$ be a local minimizer of Problem \eqref{eq:ce-loss func}. Suppose that it is not a global minimizer. This implies $(\mb W,\mb H)$ is a critical point that is not a global minimizer. According to \Cref{prop:negative curva}, the Riemannian Hessian at  $(\mb W,\mb H)$ has negative curvature. This contradicts with the fact that $(\mb W,\mb H)$ is a local minimizer. Thus, we concludes that any local minimizer of Problem \eqref{eq:ce-loss func} is a global minimizer in \Cref{thm:ce optim}. Moreover, according to \Cref{prop:negative curva}, any critical point of Problem \eqref{eq:ce-loss func} that is not a local minimizer is a Riemmannian strict saddle point with negative curvature. 
\end{proof}

\end{document}